\documentclass{article}

\usepackage[final]{neurips_2026}
\usepackage[utf8]{inputenc}
\usepackage[T1]{fontenc}
\usepackage{hyperref}
\usepackage{url}
\usepackage{booktabs}
\usepackage{amsfonts}
\usepackage{amsmath}
\usepackage{amssymb}
\usepackage{amsthm}
\usepackage{nicefrac}
\usepackage{microtype}
\usepackage{xcolor}
\usepackage{graphicx}
\usepackage{subcaption}
\usepackage[export]{adjustbox}
\newtheorem{proposition}{Proposition}

\title{The Routing and Filtering Structure of Attention}

\author{%
  Shafayeth Jamil \quad Rehan Kapadia \\
  Ming Hsieh Department of Electrical and Computer Engineering \\
  University of Southern California \\
  Los Angeles, CA 90089 \\
  \texttt{sjamil@usc.edu, rkapadia@usc.edu} \\
}

\makeatletter
\def\@noticestring{}
\makeatother

\begin{document}
\maketitle

\begin{abstract}
The attention interaction matrix $QK^{\top}$ contains two entangled computations: a skew-symmetric component that redistributes information between positions (routing) and a symmetric component that scales mutual relevance (filtering). We decompose 1{,}776 heads across five pretrained transformers and find routing operating at low rank, well below the routing capacity allocated by the weight kernel. We introduce $S$--$D$ attention as a diagnostic parameterization that disentangles routing from filtering by construction with guaranteed stability ($\mathrm{Re}(\lambda) \le 0$) and trains stably without layer normalization. When disentangled and unnormalized, routing self-organizes into a spectral cascade, effective rank $2$ at the first layer, expanding with depth across six scales from 7M to 355M parameters. The cascade predicts where attention can be simplified: linearizing the first seven layers of 125M $S$--$D$ attention costs ${<}5\%$ perplexity, whereas standard attention collapses under the same intervention. The linearizable region widens with depth. Replacing the first four layers with ELU+1 linear attention reaches within $1.4\%$ of baseline at full head dimension. Cascade-allocated architectures trade attention parameters for perplexity (47--65\% fewer attention parameters at $+3.9\%$ to $+8.4\%$ PPL). The routing--filtering decomposition makes the spectral budget legible; the cascade makes it actionable.
\end{abstract}

\section{Introduction}

Every classic transformer allocates attention uniformly across depth [1]. Each layer receives the same head dimension, the same number of parameters, the same $O(N^2)$ computation regardless of what that layer actually does. No principle exists to guide non-uniform allocation because no measurement exists to motivate it. The consequences are visible in every recent architecture. Nemotron-H places full attention in roughly $8\%$ of layers and fills the rest with linear recurrences [2]. Jamba interleaves attention and Mamba [7] at a $1{:}7$ ratio [3]. Griffin, Zamba, RecurrentGemma [4, 5, 6], each proposes a different split between quadratic and linear layers, discovered through extensive sweeps over interleaving patterns. The designs work but the methodology is search, not measurement. Each group finds a good ratio for their scale and training budget, with no guarantee that it transfers to the next.

The problem is not that attention wastes capacity. The problem is that attention hides where its capacity goes. The interaction matrix $QK^{\top}$ entangles two computations in a single unconstrained matrix: a skew-symmetric component (routing) that routes information directionally between positions, and a symmetric component (filtering) that filters by evaluating mutual relevance. These two functions have different rank requirements, different depth profiles, and different computational costs; but standard attention does not separate them. The result is that both are invisible from outside the mechanism.

We apply this decomposition to 1{,}776 heads across five pretrained transformers which reveals a consistent pattern. Routing operates at low rank on real text, well below the capacity allocated by the weight kernel. The kernel allocates capacity with a depth profile; the heads do not exercise that structure. Routing and filtering compete for the same matrix, and filtering dominates. To reveal what the spectral structure would be if it could organize freely, we separate the two functions by construction. $S$--$D$ attention replaces the unconstrained $QK^{\top}$ with $L = S - D$, where $S$ is a learned skew-symmetric matrix and $D$ is a learned positive diagonal. This guarantees $\mathrm{Re}(\lambda) \le 0$ for every eigenvalue; a stability condition strong enough to eliminate layer normalization entirely. The model trains stably at 355M parameters with no normalization of any kind.

When routing and filtering are disentangled and no normalization equalizes representations across depth, routing self-organizes into a rank hierarchy or spectral cascade. First layer converges to rank two; a single rotation plane that broadly redistributes information across the sequence. The final layer expands to ten or more independent rotation planes. This cascade appears at every scale we test, from 7M to 355M parameters. The fraction of the network operating at minimum spectral rank grows with scale and with it, the fraction that can be replaced with cheaper mechanisms. Filtering collapses in the opposite direction, from a high-rank landscape of pairwise mutual relevance to a single scalar per head. When routing is properly organized, one number per head is all the filtering the model needs.

The ascending cascade is not a property of $S$--$D$ attention. It is the optimizer's natural solution when neither entanglement (standard attention) nor LN-induced equalization ($S$--$D$+LN) is present. And once measured, it is actionable. The spectral rank at each layer prescribes exactly what hybrid architectures have been searching for: which layers need full rank quadratic attention, which can be replaced with low-rank or linear mechanisms.

\section{The Routing--Filtering Decomposition}

For any attention head with query and key projections $W_Q, W_K \in \mathbb{R}^{d_{\mathrm{model}} \times d_{\mathrm{head}}}$, the pre-softmax interaction on input sequence $X \in \mathbb{R}^{N \times d_{\mathrm{model}}}$ is:
\begin{equation*}
A_h = (X W_Q W_K^{\top} X^{\top}) / \sqrt{d_h} \in \mathbb{R}^{N \times N}.
\end{equation*}
Entry $A[i,j]$ is the score that token $i$ assigns to token $j$. This matrix decomposes uniquely into two components:
\begin{align*}
F[i,j] &= (A[i,j] + A[j,i])/2 \quad &\text{symmetric: } F = F^{\top}, \\
R[i,j] &= (A[i,j] - A[j,i])/2 \quad &\text{skew-symmetric: } R = -R^{\top}.
\end{align*}

Each entry of $F$ is the average interaction between two tokens: how mutually relevant they are, regardless of direction. $F[i,j] = F[j,i]$: if token $i$ finds token $j$ relevant, token $j$ finds token $i$ equally relevant. This is an undirected similarity score. Each entry of $R$ is the directional imbalance: how much more one token attends to the other than the reverse. $R[i,j] = -R[j,i]$: any directional preference in one direction is matched by an equal and opposite preference in the other. Every attention score is the sum of a mutual part and a directional part.

The singular value decomposition of each component reveals how many independent patterns each contains, and the two components have fundamentally different spectral structure.

\paragraph{Filtering modes.} $F$ is symmetric, so its SVD pairs each singular value $\sigma_k$ with a single vector $u_k$ over token positions. The outer product $\sigma_k u_k u_k^{\top}$ adds mutual scores between every pair of tokens that participate in that pattern. For example, $\sigma_1 = 3.0$ with $u_1$ supported on three tokens produces equal mutual scores between every pair of those three tokens at gain $3.0$. This is why we call $F$ the filtering component. It selectively amplifies and suppresses relevance modes, exactly as a bank of independent gain channels does.

\paragraph{Routing modes.} $R$ is skew-symmetric, so each singular value requires two vectors: a source pattern and a sink pattern. This is the fundamental structural difference. Consider a 6-token sequence: ``The cat sat on the mat''. Suppose: $\sigma_1 = 2.0$, $u_1 = [1,0,0,0,0,0]$ (sink), $v_1 = [0,1,0,0,0,0]$ (source). The contribution $\sigma_1(u_1 v_1^{\top} - v_1 u_1^{\top})$ produces $R[0,1] = +2.0$ and $R[1,0] = -2.0$. The $+2.0$ that token 0 gains in attending to token 1 is exactly the $-2.0$ that token 1 loses in attending to token 0. Zero sum per pair. $R$ redirects attention mass without creating or destroying it. Its eigenvalues are purely imaginary; it generates rotation between modes with zero amplitude change. This is conservative transport: information moves but nothing is created or destroyed. A second flow ($\sigma_2 = 0.8$) at lower gain adds a weaker, independent directional current. Because every flow requires a source and a sink, $R$ always has even rank.

The rank of $R$ counts the number of independent rotation planes: the routing capacity. The rank of $F$ counts the number of independent gain channels: the filtering capacity. We define the routing--filtering ratio: $\rho_h = \|R_h\|_F / \|F_h\|_F$. When $\rho > 1$, the head is primarily routing. When $\rho < 1$, the head is primarily filtering. We measure the spectral complexity of each component via the effective rank: $\mathrm{effrank}(M) = (\sum_i \sigma_i) / \sigma_{\max}$, which counts the number of singular values that contribute meaningfully to the matrix. A matrix with one dominant singular value has effective rank ${\approx}1$; a matrix with $k$ equal singular values has effective rank $k$.

\section{Pretrained Model Analysis}

We decompose the attention interaction matrix of every head in GPT-2 (12 layers, 12 heads) [8], GPT-2 Medium (24 layers, 16 heads), GPT-2 Large (36 layers, 20 heads), BERT-base (12 layers, 12 heads) [9], and Pythia-410M (24 layers, 16 heads) [22]; 1{,}776 heads total. We perform two complementary measurements. The first is sequence-level: extract $A = QK^{\top}/\sqrt{d}$ on six text sequences of varying length and domain (Appendix~\ref{app:sequences}), decompose into $R + F$, and report $\rho$, $\max \mathrm{Re}(\lambda)$, and effective rank averaged across sequences. This measures what each head does on real text. The second is weight-level: form the per-head kernel $M = W_Q^{\top} W_K / \sqrt{d}$ directly from the projection matrices and decompose into its skew and symmetric parts. This measures the kernel structure independent of input.

\begin{figure}[!t]
\centering
\begin{subfigure}{0.24\linewidth}
  \centering
  \includegraphics[height=4cm, keepaspectratio]{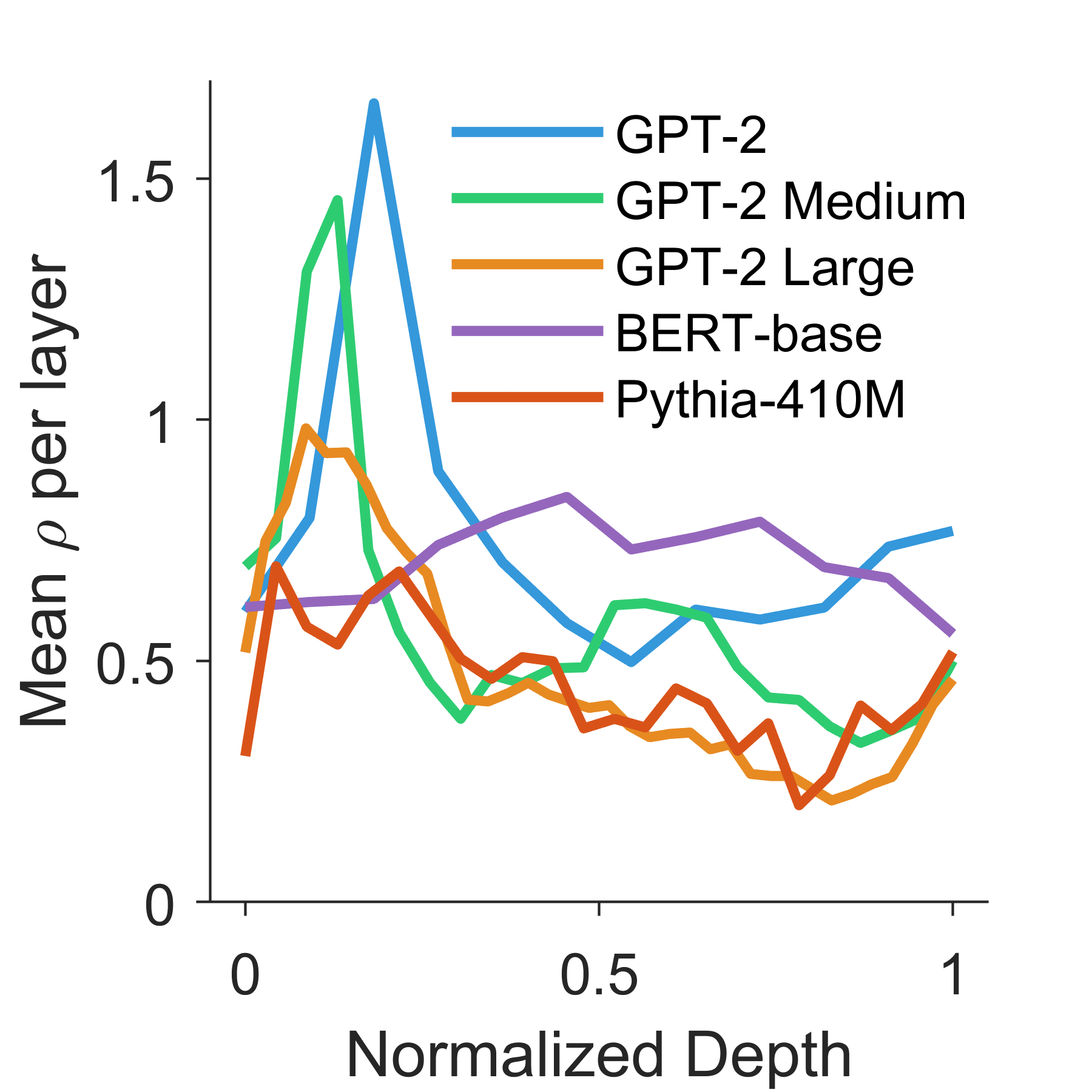}
\end{subfigure}
\hfill
\begin{subfigure}{0.24\linewidth}
  \centering
  \includegraphics[height=4cm, keepaspectratio]{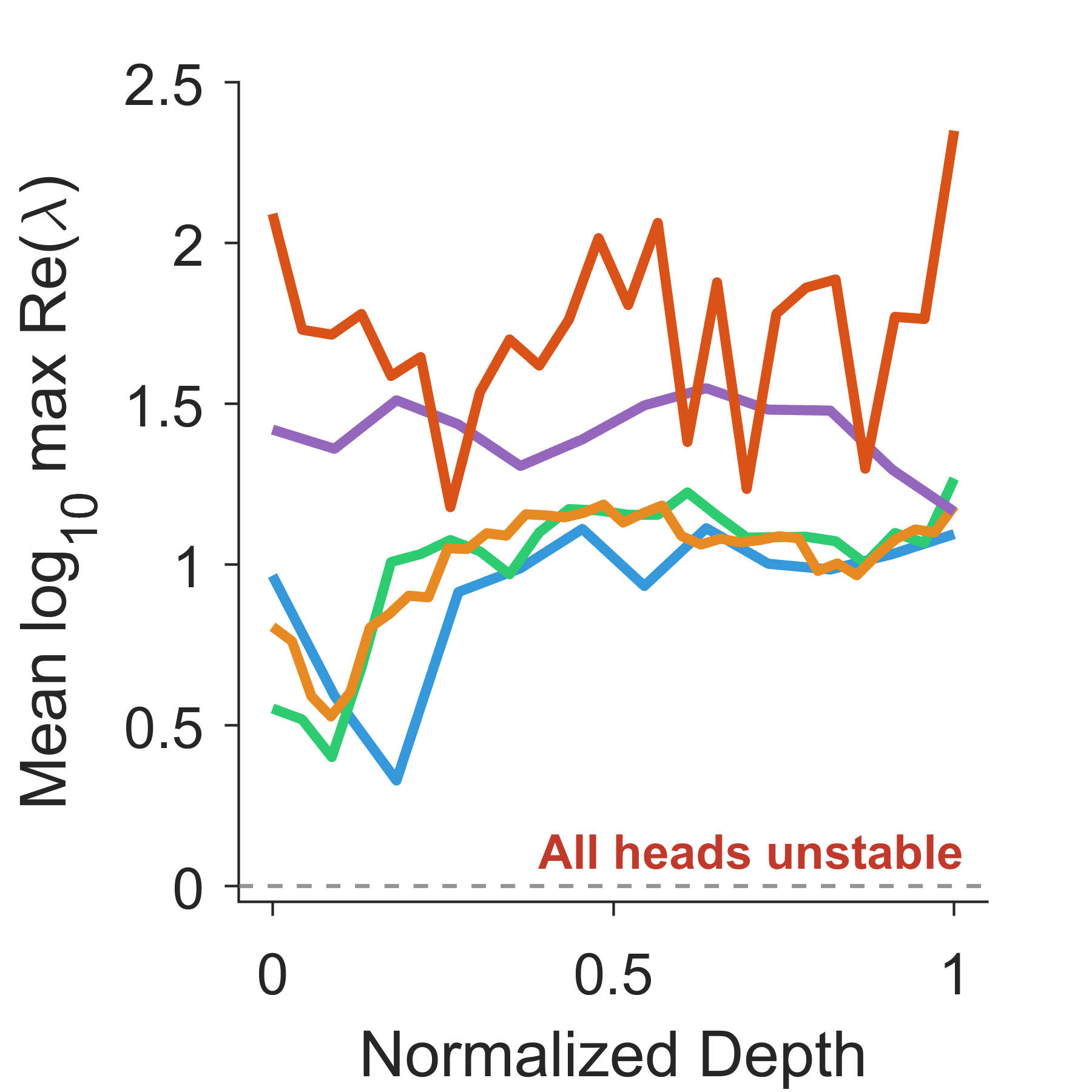}
\end{subfigure}
\hfill
\begin{subfigure}{0.24\linewidth}
  \centering
  \includegraphics[height=4cm, keepaspectratio]{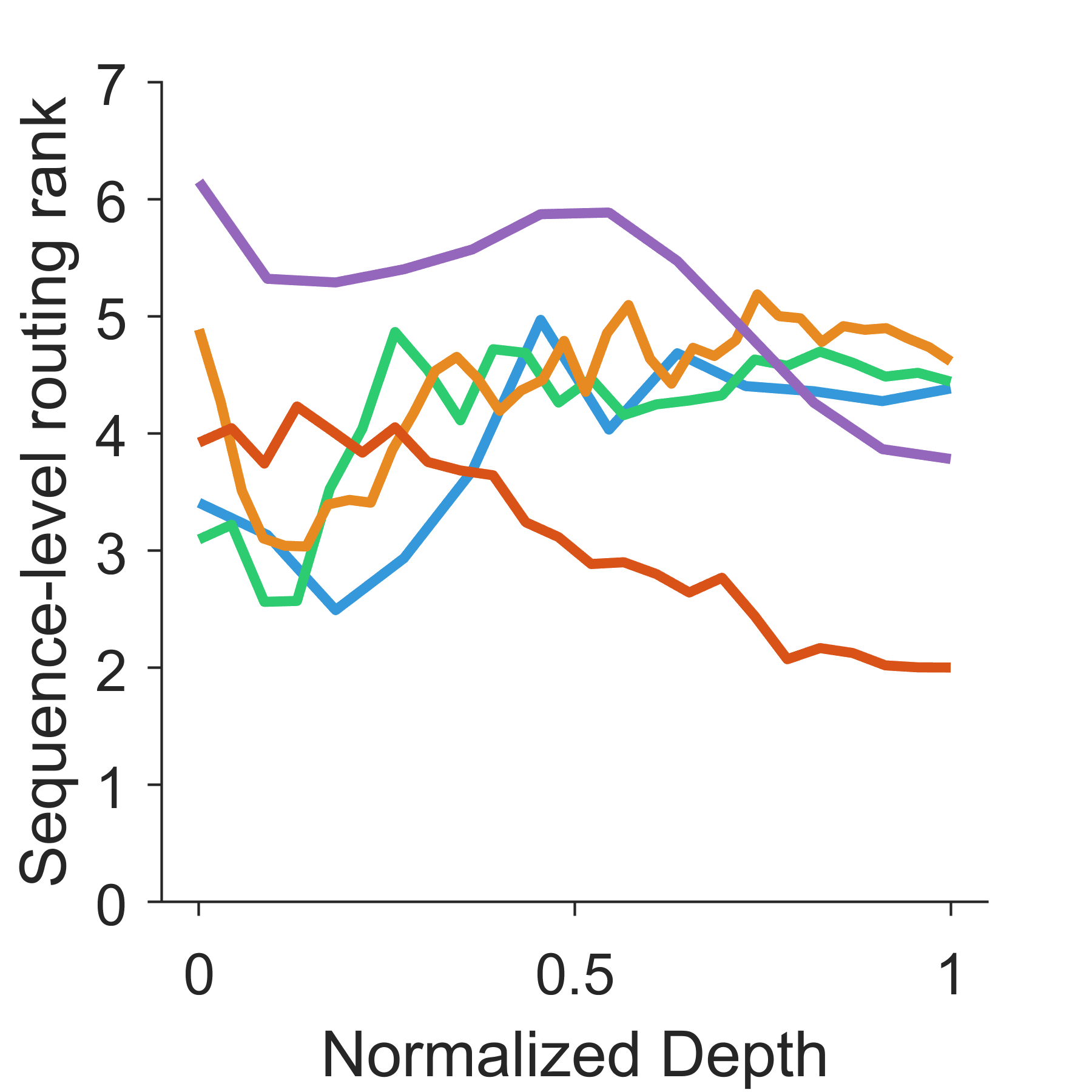}
\end{subfigure}
\hfill
\begin{subfigure}{0.24\linewidth}
  \centering
  \includegraphics[height=4cm, keepaspectratio]{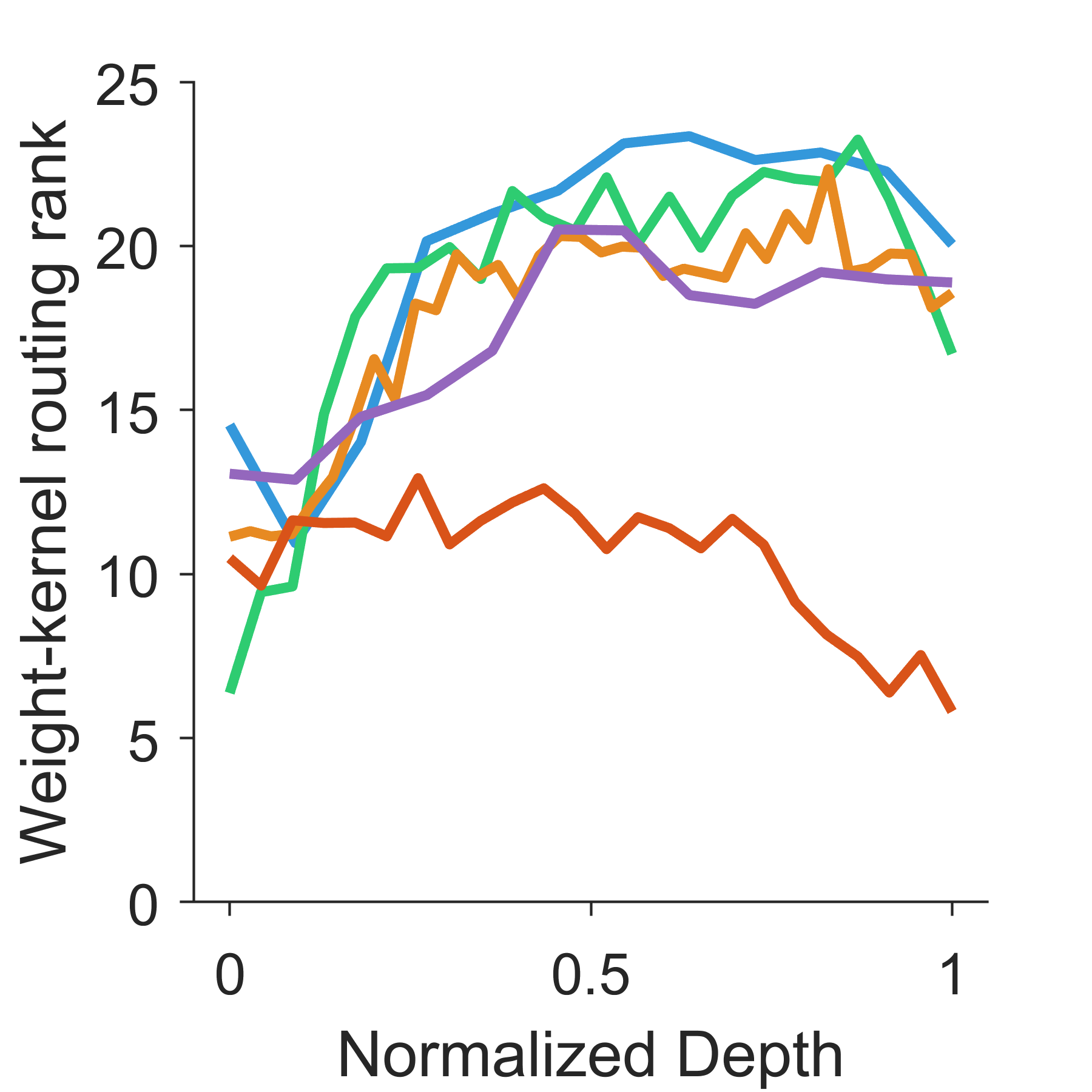}
\end{subfigure}
\caption{The routing--filtering structure of pretrained models. (a) Mean $\rho$ per layer across five models. (b) All 1{,}776 heads are dynamically unstable ($\max \mathrm{Re}(\lambda) > 0$). (c) Sequence-level effective routing rank per layer. (d) Weight-kernel effective routing rank per layer. The kernel allocates 6--23 dimensions per head, $2$--$7\times$ larger than the sequence-level rank in (c).}
\label{fig:pretrained}
\end{figure}

Three observations follow.

\paragraph{Filtering dominates routing in standard attention.} Only 48 of 720 heads in GPT-2 Large exceed $\rho = 1$; in Pythia-410M, 17 of 384 heads exceed $\rho = 1$. The median $\rho$ across GPT-2 variants is below $0.7$, in Pythia it is $0.39$, and in BERT-base it is approximately $0.7$. The GPT-2 variants and Pythia share a descending $\rho$ profile across depth with an elevation at the final layers (Fig.~\ref{fig:pretrained}a). BERT peaks in the middle layers. The shape of the $\rho$ curve distinguishes architectural families from the interaction matrix alone.

\paragraph{All heads are dynamically unstable.} Every one of the 1{,}776 heads has $\max \mathrm{Re}(\lambda) > 0$ (Fig.~\ref{fig:pretrained}b). Instability is bounded in the GPT variants and BERT, but grows with depth in Pythia-410M reaching $\max \mathrm{Re}(\lambda) \sim 10^4$. Standard attention relies heavily on layer normalization to contain this.

\paragraph{Routing is low-rank when used, higher-rank when measured at the weights.} On real text, the sequence-level effective rank of R is 2.0–5.2 across GPT-2 variants and Pythia, 3.8–6.2 in BERT (Fig.~\ref{fig:pretrained}c). Weight-kernel rank is 2–7× larger (Fig.~\ref{fig:pretrained}d), exposing a kernel–function gap. The depth profile differs across architecture and is consistent between the two measurement levels for each model. GPT-2 variants dip at sequence-level in early-middle layers and rise to a plateau, with a kernel that rises into mid-network and plateaus. BERT runs higher throughout at both levels. Pythia organizes a descending profile at both levels: sequence rank falls from $3.92$ at layer 0 to $2.00$ at layer 23, and weight-kernel rank falls from 10--13 across the first eighteen layers to 5--8 in the terminal layers. Despite different shapes, flat (GPT-2), elevated (BERT), descending (Pythia), every standard architecture organizes routing into some structured depth profile.

The findings are then: filtering carries most of the spectral budget in standard attention, routing operates well below kernel capacity, and every head is dynamically unstable. Section~\ref{sec:sd} introduces a parameterization that addresses the third by construction. Section~\ref{sec:cascade} measures the first two when filtering and routing are no longer entangled: finding 1 inverts (Section~\ref{sec:filtering-collapse}: filtering collapses to a single scalar  per head), and finding 2 develops an ascending depth structure that was invisible in standard attention (Section~\ref{sec:routing-cascade}: routing organizes into a depth-cascade).

\section{Separating Routing from Filtering: \texorpdfstring{$S$--$D$}{S-D} Attention}\label{sec:sd}

Standard attention entangles routing and filtering in a single unconstrained matrix. We separate them by construction.

\paragraph{Conservative routing ($S$).} Each head learns query and key projections $W_Q, W_K \in \mathbb{R}^{d_{\mathrm{model}} \times d_{\mathrm{head}}}$. The raw interaction $P = Q_s K_s^{\top} / \sqrt{d}$ yields the skew-symmetric part: $S = (P - P^{\top})/2$. $S$ is skew-symmetric by construction. Its eigenvalues are purely imaginary. This is the routing component: directional, conservative, zero-sum. Each singular value pair defines one rotation plane which translates to one independent source--sink flow through the sequence.

\paragraph{Dissipative filtering ($D$).} Each token $i$ computes a per-head damping coefficient:
\begin{equation*}
d_i = \mathrm{softplus}(W_d \cdot x_i + b_d) + \varepsilon,
\end{equation*}
where $W_d \in \mathbb{R}^{d_{\mathrm{model}} \times H}$ and $b_d$ is a learnable bias. The filtering component is the diagonal matrix $D = \mathrm{diag}(d_1, \ldots, d_N)$. Every entry is strictly positive which is a deliberate constraint. Standard attention allocates up to $N$ independent filtering modes per head; $S$--$D$ allocates one scalar per token per head (Section~\ref{sec:cascade}). $\varepsilon$ is a fixed offset that keeps eigenvalues bounded away from the imaginary axis during training. The constraint $\mathrm{Re}(\lambda) \le 0$ holds for any $\varepsilon \ge 0$; the offset is a conditioning parameter, not a tunable hyperparameter. Significance of $\varepsilon$ on stability is characterized in Appendix~\ref{app:epsilon}.

\begin{proposition}[Architectural stability]\label{prop:stab}
Let $L = S - D$ where $S \in \mathbb{R}^{N \times N}$ is skew-symmetric and $D \in \mathbb{R}^{N \times N}$ is a real diagonal matrix with $D_{ii} \ge 0$. Then every eigenvalue $\lambda$ of $L$ satisfies $\mathrm{Re}(\lambda) \le 0$.
\end{proposition}

\begin{proof}
For eigenvector $v \in \mathbb{C}^N$ with eigenvalue $\lambda$: $\lambda \|v\|^2 = v^* L v = v^* S v - v^* D v$. Since $S$ is real and skew-symmetric, $v^* S v$ is purely imaginary, so $\mathrm{Re}(v^* S v) = 0$. Since $D$ is real and positive-semidefinite (diagonal with non-negative entries), $v^* D v \ge 0$. Hence $\mathrm{Re}(\lambda) \|v\|^2 = -v^* D v \le 0$.
\end{proof}

The interaction matrix $L = S - D$ is therefore dissipative. The interaction cannot amplify; it can only conserve (through $S$) or dissipate (through $D$). This is the condition that layer normalization [10] enforces approximately in standard transformers. $S$--$D$ attention enforces it exactly, from the parameterization alone. Attention weights are computed as $\mathrm{softmax}(L) \cdot V$ with the causal mask applied to $L$ before softmax. The architecture is otherwise identical to a standard transformer: residual connections, feedforward layers, weight tying. The only modification is the attention interaction. This constraint was introduced for physical dynamical systems [11, 12].

\paragraph{Training.} We train three configurations at 125M parameters ($d_{\mathrm{model}} = 768$, 12 layers, 12 heads, sequence length 1024) on OpenWebText [13] (${\sim}1$B tokens, 2 epochs): standard attention with LayerNorm (LN), $S$--$D$ with LN, and $S$--$D$ without any normalization. All use AdamW (lr=$6 \times 10^{-4}$, $\beta_1=0.9$, $\beta_2=0.95$, weight decay $0.1$) with cosine schedule (min lr $= 6 \times 10^{-5}$), 2{,}000-step warmup, and bfloat16 autocast.

All three configurations train stably (Fig.~\ref{fig:training}). The three val PPLs decompose the cost of the architectural choices: Standard at $27.87 \pm 0.01$, $S$--$D$ with LN at $31.14 \pm 0.12$ (a $3.27$ PPL gap, $+11.7\%$), $S$--$D$ without LN at $37.39 \pm 0.09$ (a further $6.25$ PPL gap, $+20.1\%$ relative to $S$--$D$+LN).

\begin{table}[!h]
\caption{Comparison of Standard, $S$--$D$+LN, and $S$--$D$ (no LN) models.}
\label{tab:training}
\centering
\begin{tabular}{lccc}
\toprule
& Standard & $S$--$D$+LN & $S$--$D$ (no LN) \\
\midrule
Model parameters & 124{,}412{,}160 & 124{,}513{,}680 & 124{,}475{,}280 \\
Max activation & $39.5$ & $36.75$ & $56.5$ \\
Final val PPL & $27.87 \pm 0.01$ & $31.14 \pm 0.12$ & $37.39 \pm 0.09$ \\
PPL gap vs Standard (\%) & --- & $+11.7\%$ & $+34.2\%$ \\
\bottomrule
\end{tabular}
\end{table}

Universal instability (finding 3) is reduced substantially: mean $\max \mathrm{Re}(\lambda)$ of per-head weight kernel M drops from $0.75$ in Standard to $0.034$ in $S$--$D$ no-LN, a $22\times$ reduction. The remaining two findings, filtering dominance and the kernel--function gap, are addressed in Section~\ref{sec:cascade}.

\section{The Spectral Cascade}\label{sec:cascade}

The $S$--$D$ model without LN is the least performant configuration we train. It is also the most informative. Its weights reveal two findings that are not present in standard attention and obscured by LN.

\subsection{Filtering collapses to a scalar}\label{sec:filtering-collapse}

We take the trained 125M $S$--$D$ model (baseline PPL $37.39$, overlapping-stride PPL $34.99$), intercept the interaction matrix $L = S - D$ at runtime, and truncate each component independently to a target rank via SVD. No retraining is done. Replacing every per-token damping $d_i^{(h)}$ with its per-head mean $\bar{d}^{(h)}$ at all 12 layers costs $+0.3$ PPL ($34.99 \to 35.29$). The entire filtering component reduces to 144 scalars (one per head, per layer). Removing routing entirely ($S = 0$) yields PPL $699.40$, indicating the mechanism is routing dominated.

\subsection{Routing organizes into a spectral cascade}\label{sec:routing-cascade}

\begin{figure}[!t]
\centering
\begin{subfigure}{0.24\linewidth}
  \centering
  \includegraphics[height=3.9cm, keepaspectratio]{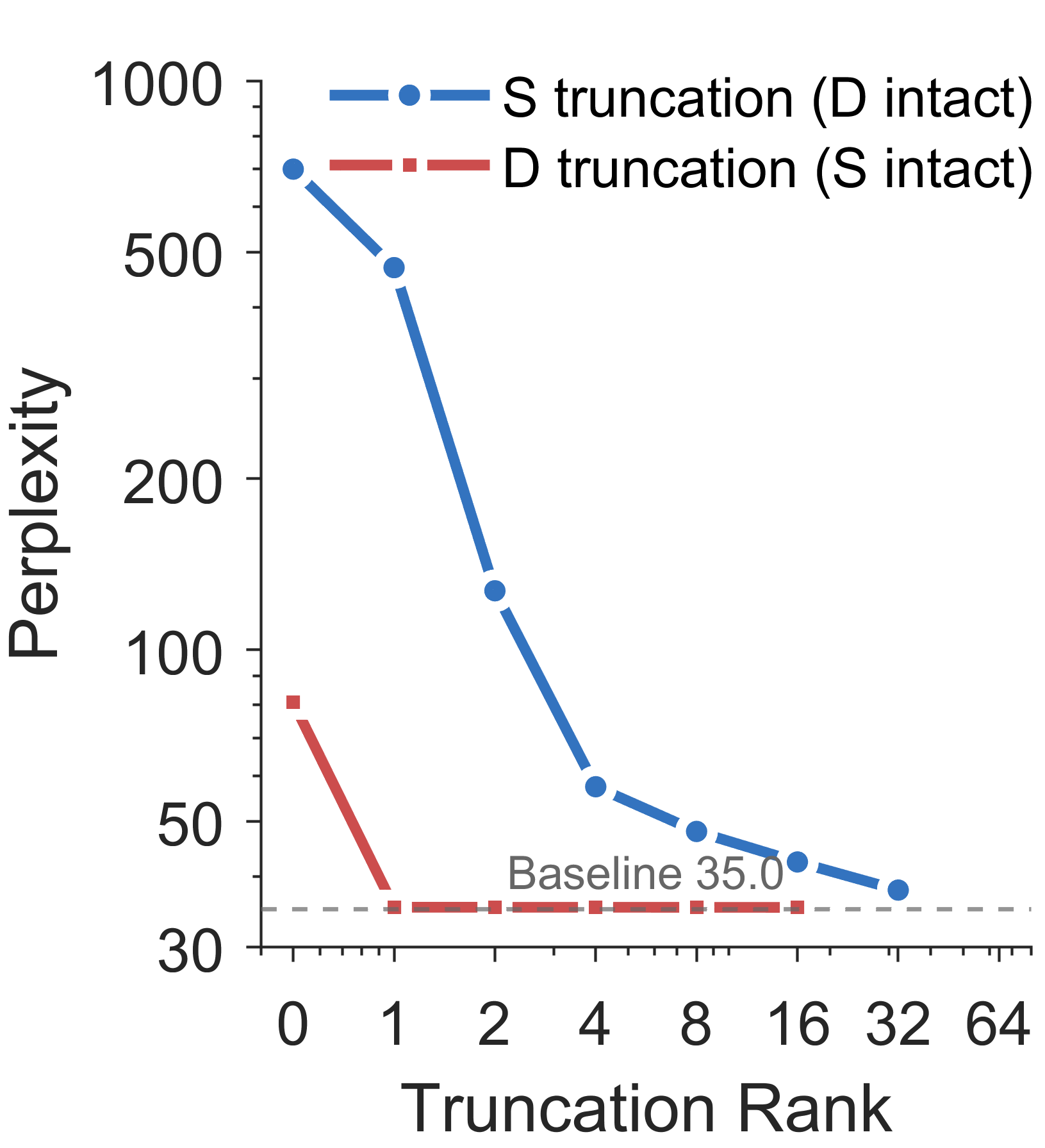}
\end{subfigure}
\hfill
\begin{subfigure}{0.24\linewidth}
  \centering
  \includegraphics[height=3.9cm, keepaspectratio]{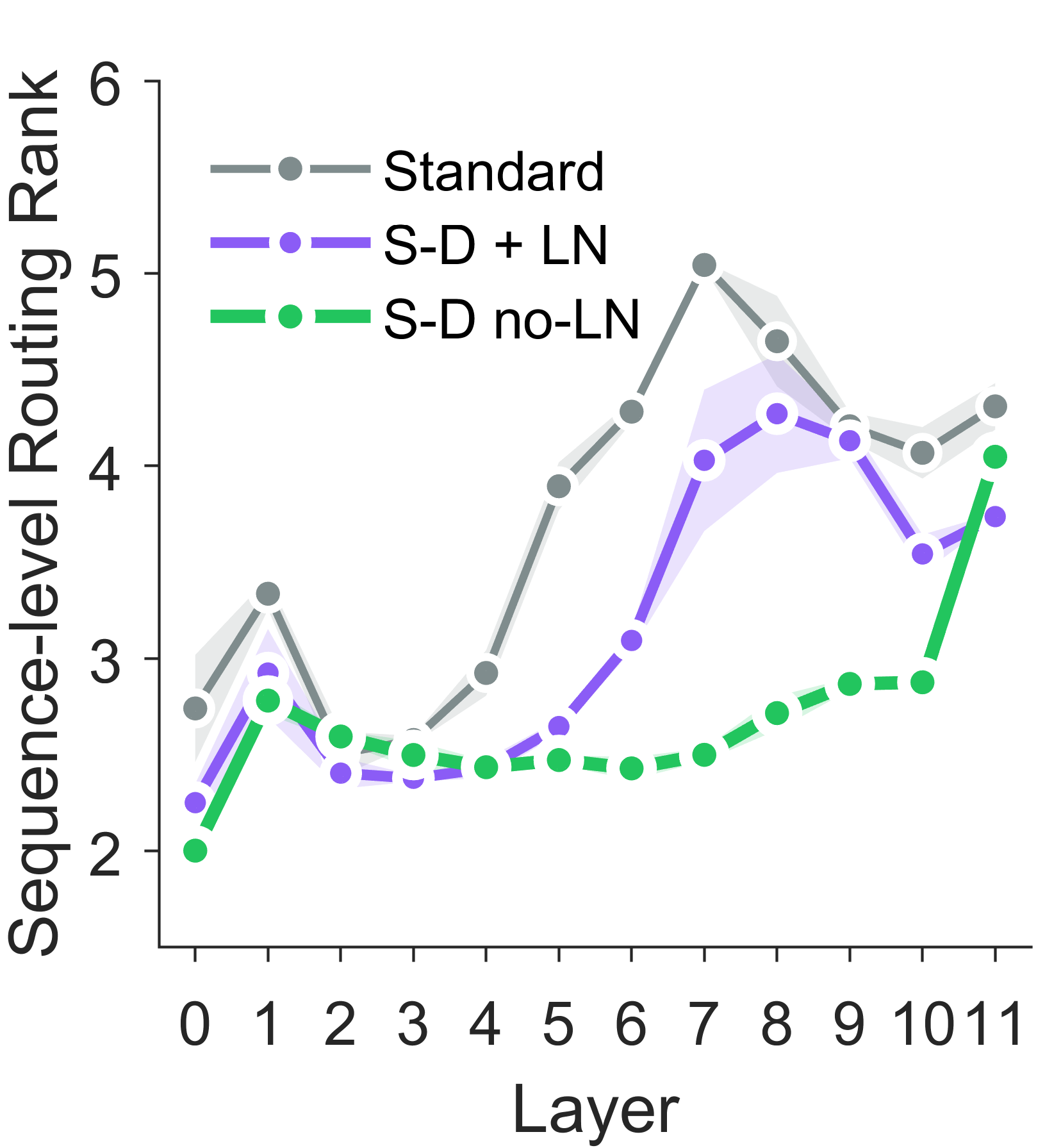}
\end{subfigure}
\hfill
\begin{subfigure}{0.24\linewidth}
  \centering
  \includegraphics[height=3.9cm, keepaspectratio]{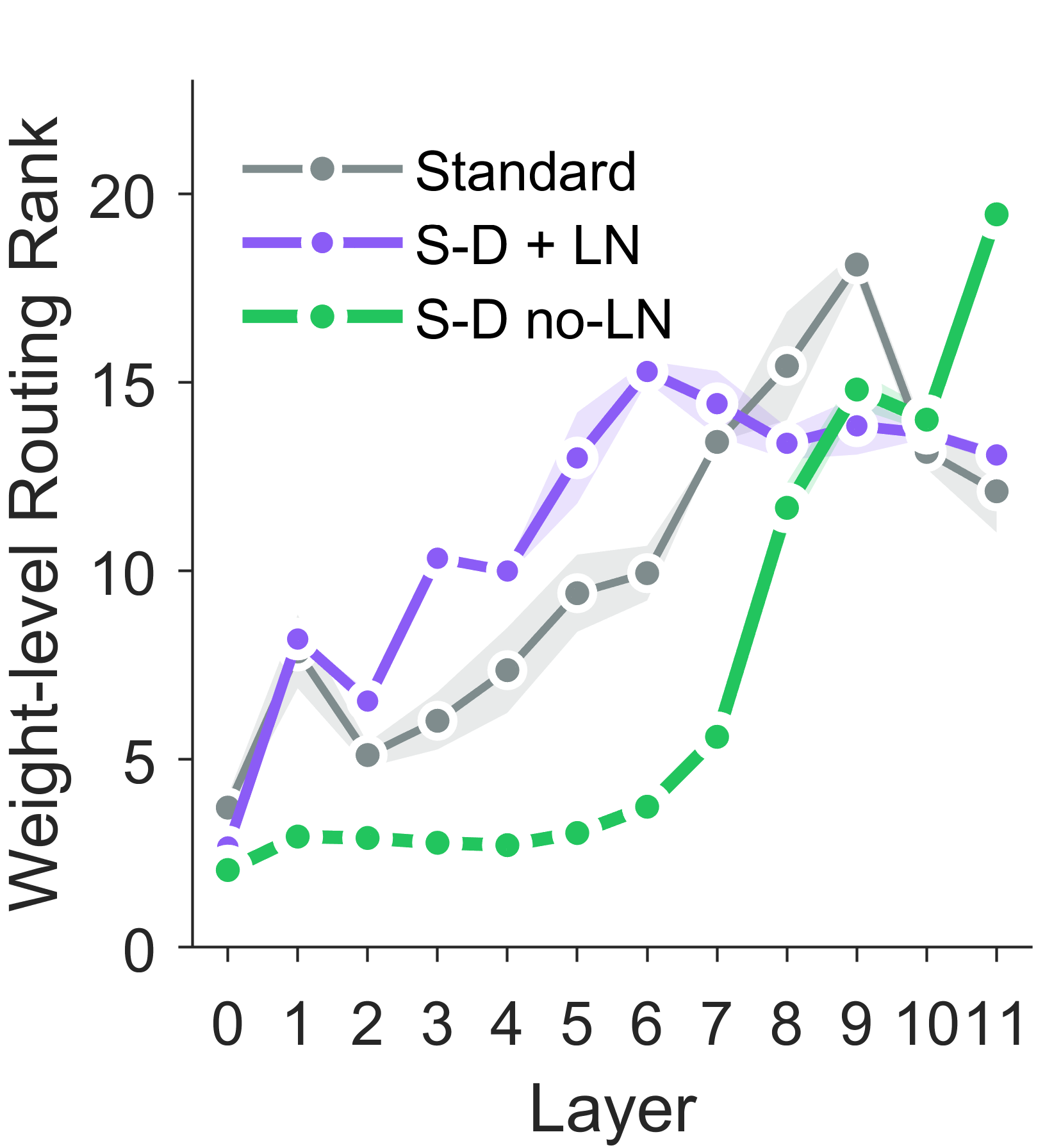}
\end{subfigure}
\hfill
\begin{subfigure}{0.24\linewidth}
  \centering
  \includegraphics[height=3.9cm, keepaspectratio]{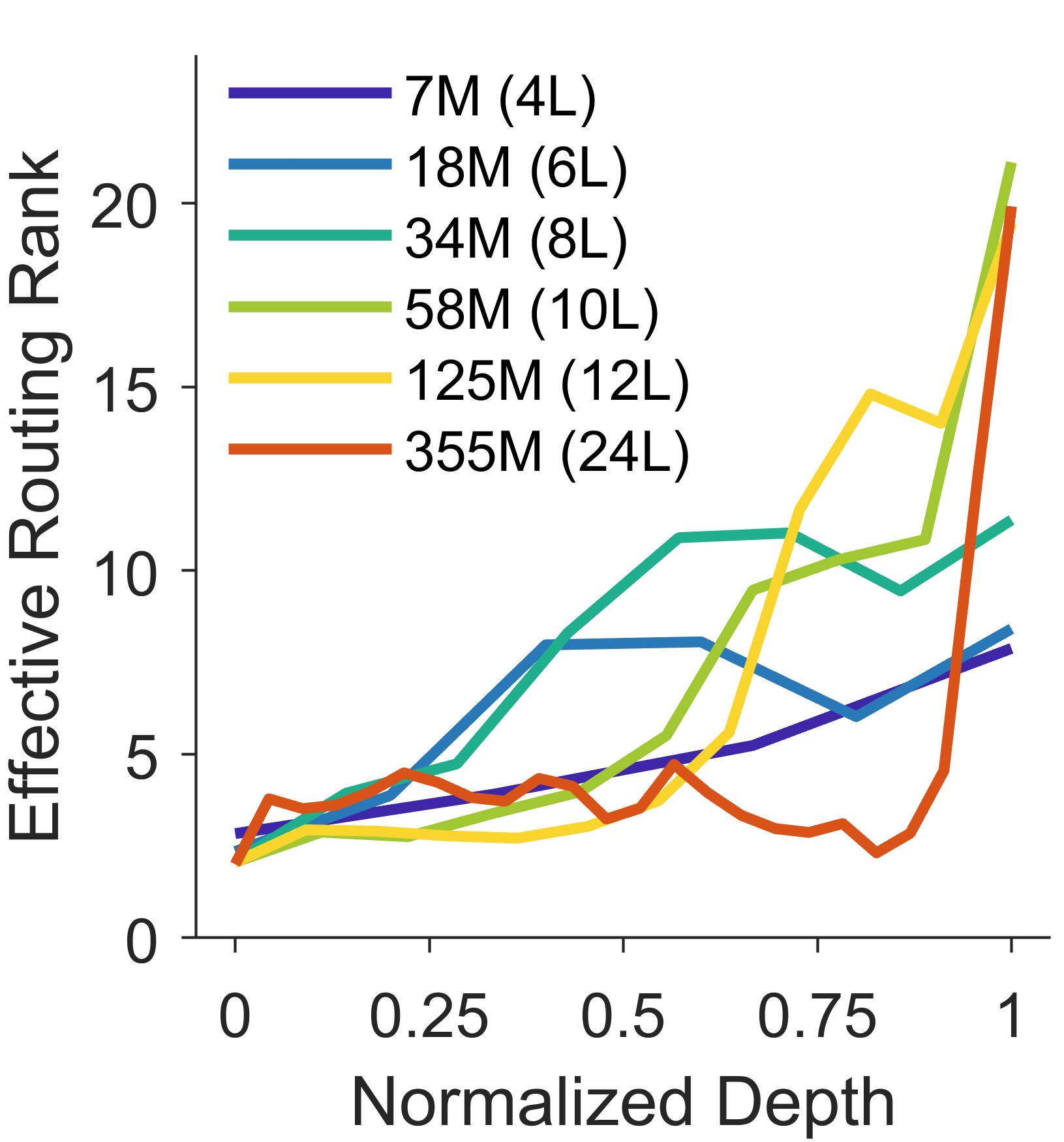}
\end{subfigure}
\caption{The spectral cascade. (a) Spectral budget of $S$--$D$ attention (125M, baseline PPL $34.99$). Truncating $D$ to rank $1$ costs $+0.3$ PPL; removing $S$ costs $20\times$. (b) Sequence-level routing rank per layer at 125M. (c) Weight-level routing rank per layer at 125M. (d) The weight-level cascade across scale. Six models from 7M to 355M parameters. Layer 0 is always rank $2$. The fraction of the network at lower rank grows with depth.}
\label{fig:cascade}
\end{figure}

With filtering collapsed, we examine what routing does with its freedom. At 125M, $S$--$D$ no-LN organizes routing into a near-monotonic depth profile (Fig.~\ref{fig:cascade}b): layer 0 collapses to effective rank $2.00 \pm 0.001$ across three random seeds, layers 1--10 form a flat plateau at rank 2.4--2.9, and layer 11 spikes to $4.05 \pm 0.04$. Layer 0 learns a single rotation plane; one source--sink flow that broadly redistributes information across the sequence. Each successive layer adds rotation planes. Layer 11 operates with several independent directional flows. Standard attention at 125M shows sequence-level rank rises from $2.74$ at layer 0 to a peak of $5.04$ at layer 7, then falls to $4.31$ at the terminal layer. $S$--$D$ with LN produces the same qualitative shape across two seeds, peaking at layers 8--9 (4.1--4.5) and falling to $3.7$ at layer 11 (Fig.~\ref{fig:cascade}b). $S$--$D$ without LN is the only configuration where routing rank reaches its maximum at the terminal layer.

Weight-level effective rank at 125M shows the identical depth profile, scaled up: $S$--$D$ no-LN rises from $2.05$ at layer 0 to $19.45$ at layer 11 (Fig.~\ref{fig:cascade}c); Standard organizes a mid-network bulge peaking at layer 9 (rank ${\sim}18$) and falling to $12$ at the terminal layer; $S$--$D$+LN bulges around layer 6 (rank ${\sim}15$) and falls to $13$ at layer 11. The kernel allocates up to $5\times$ more capacity than the function exercises, but the shape of the depth allocation is what each architecture has organized.

We train $S$--$D$ models without LN at six scales: 7M (4 layers), 18M (6 layers), 34M (8 layers), 58M (10 layers), 125M (12 layers), and 355M (24 layers). Three regularities hold across all six scales (Fig.~\ref{fig:cascade}d). First, layer 0 effective rank converges toward $2$ with scale: $2.84$ at 7M, $2.06$ at 58M, and $2.00$ at 355M. Second, the terminal layer reaches rank 19--21 at $d_{\mathrm{head}} = 64$ and rank $\sim8$ at $d_{\mathrm{head}} = 32$, indicating the terminal rank is set by head dimension, not by depth. Third, the fraction of layers operating at rank $\le 5$ grows monotonically with depth: $33\%$ at 18M, $58\%$ at 125M, and $92\%$ at 355M. The plateau widens with depth; the high-rank work concentrates at the top.
The cascade is the optimizer's natural solution when neither entanglement nor LN obstructs depth asymmetry; Appendix D documents its training-time emergence.

\section{Spectral Surgery}

The spectral cascade predicts that the cost of simplifying a layer should follow its routing rank: cheap at the bottom, expensive at the top. We test this by surgically linearizing individual layers (forcing rank-2 routing and scalar filtering) and measuring perplexity without retraining. All surgery experiments use overlapping-stride PPL on OWT validation: Standard 125M = 29.6, S-D no-LN 125M = 34.99 (the latter consistent with section 5.1).

\subsection{Per-layer linearization cost}\label{sec:perlayer-surgery}

In $S$--$D$ attention (Fig.~\ref{fig:surgery}a), layers 0--5 can each be linearized at less than $0.1\%$ perplexity cost. The cost then rises monotonically with the cascade: $+1.2\%$ at layer 6, $+2.0\%$ at layer 7, $+3.7\%$ at layer 8, $+6.7\%$ at layer 9, $+11.6\%$ at layer 10, $+55.6\%$ at layer 11. Each layer's linearization cost correlates with its effective routing rank. The cascade is readable from the weights and predictive of function. In standard attention, the same linearization costs $0.1$--$13\%$ across layers with no consistent pattern. The peak cost ($12.9\%$) occurs at layer 7 not at layer 11 as the cascade in $S$--$D$ would predict; layers 10 and 11 cost only 1.5--1.7\% each. This is consistent with the mid-network bulge structure Section~\ref{sec:routing-cascade} reports for Standard at 125M: high-rank work concentrates around layer 7, so linearizing layer 7 is expensive while linearizing the terminal layers is comparatively cheap.

\subsection{Cumulative linearization}

Starting from the front of the $S$--$D$ model, layers 0 through 6 ($58\%$ of the network) can be simultaneously linearized within $5\%$ of baseline perplexity (Fig.~\ref{fig:surgery}b). The cost stays under $0.2\%$ through layer 5, jumps to $2.8\%$ at layer 6, then crosses $5\%$ at layer 7. Beyond layer 7, cumulative cost rises sharply: $+32\%$ at layer 8, $+82\%$ at layer 9, $+142\%$ at layer 10, $+270\%$ at full linearization. In standard attention, cumulative front-to-back linearization collapses by layer 3: perplexity jumps from $29.6$ to $41.7$ once the third layer is folded in ($+41\%$ above baseline). By layer 5 it nearly doubles ($50.7$), by layer 7 it triples ($101.7$), reaching $168.4$ at full linearization ($+469\%$).

\begin{figure}[!t]
\vspace{-6pt}     
\centering
\begin{subfigure}{0.4\linewidth}
  \centering
  \includegraphics[height=4.5cm, keepaspectratio]{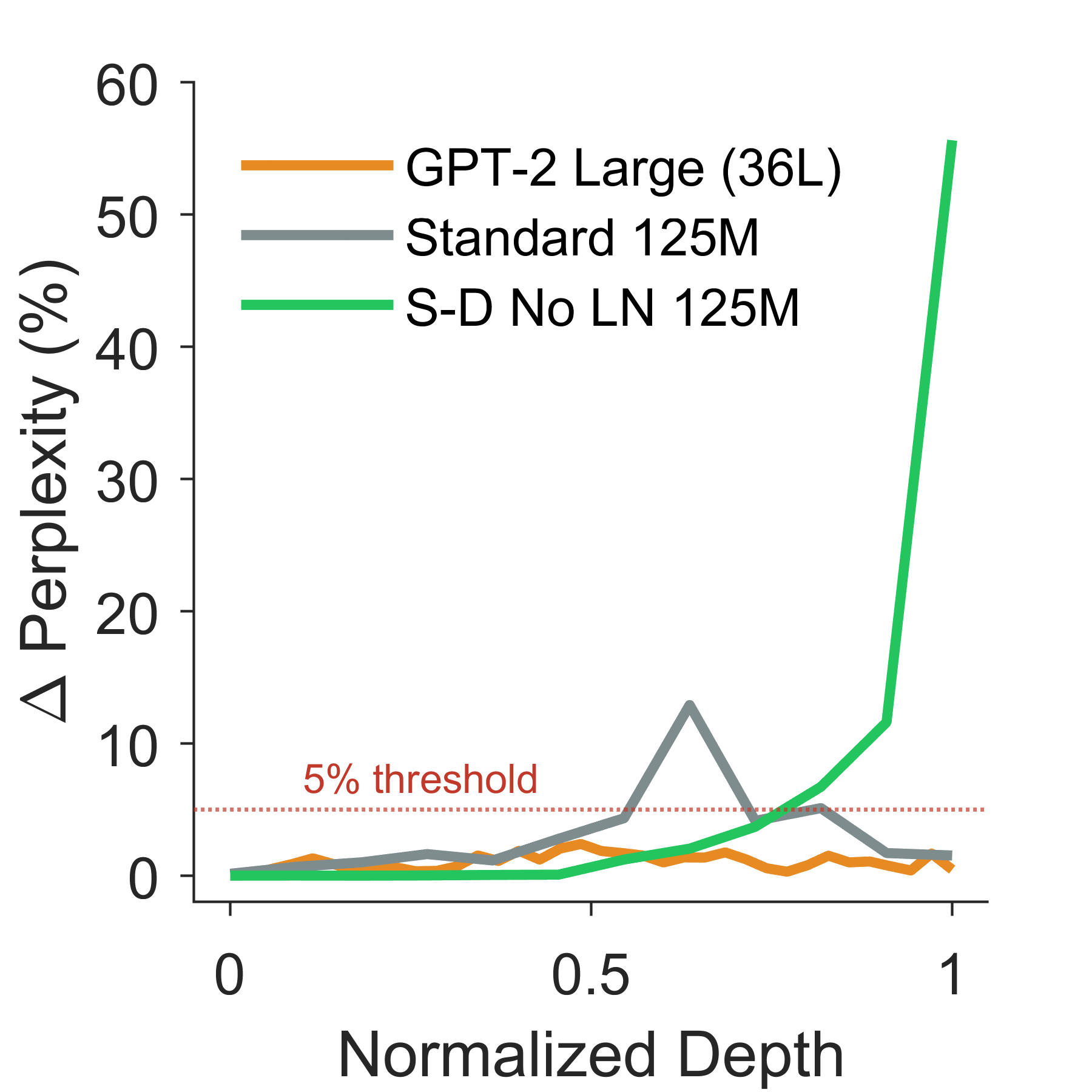}
\end{subfigure}
\hfill
\begin{subfigure}{0.4\linewidth}
  \centering
  \includegraphics[height=4.5cm, keepaspectratio]{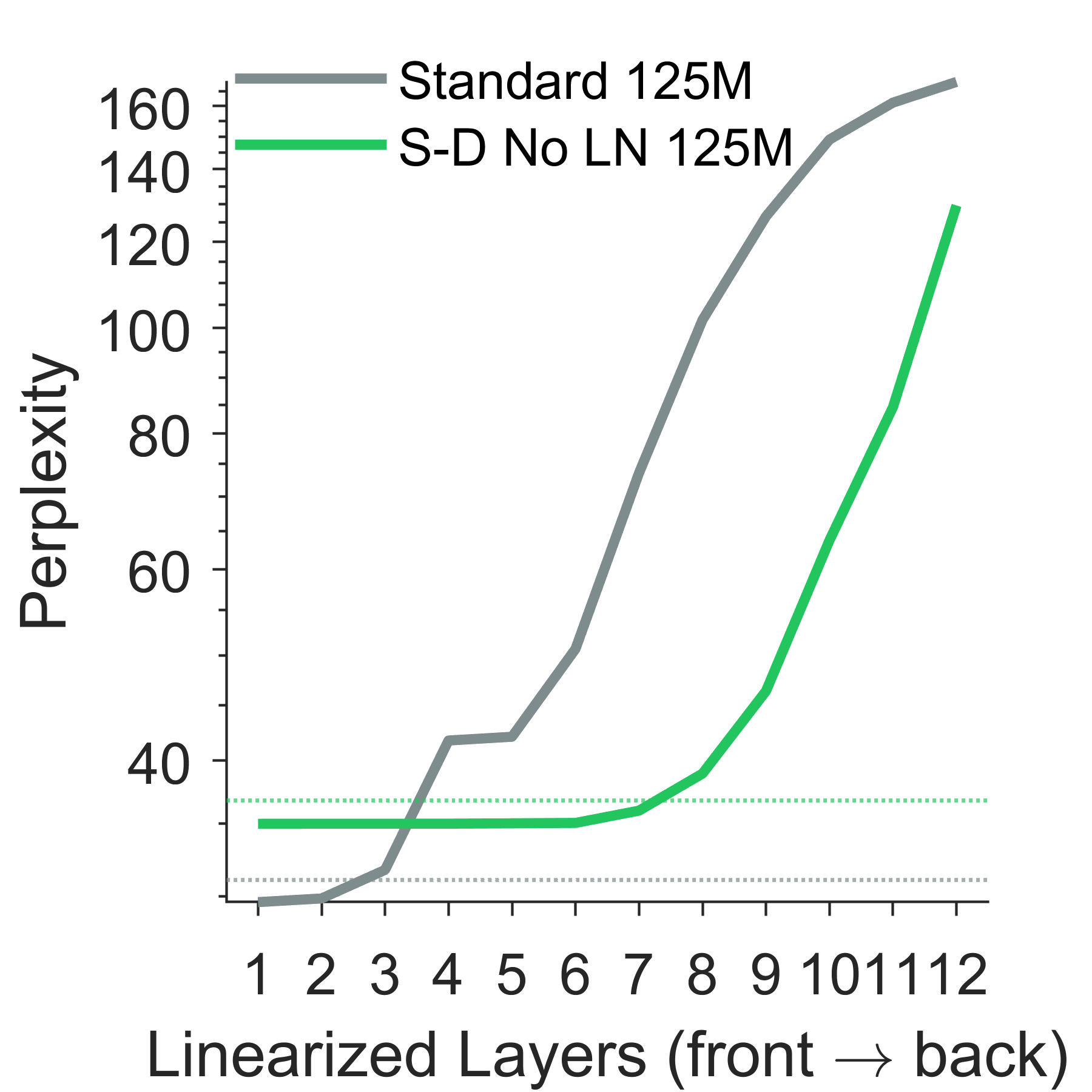}
\end{subfigure}
\caption{Linearization cost follows the cascade. (a) Per-layer $\Delta$PPL\% after forcing rank-2 routing and scalar filtering. $S$--$D$ rises with the cascade. (b) Cumulative linearization. $S$--$D$ front-to-back: 7 of 12 layers linearized within $5\%$ of baseline. Standard front-to-back collapses by layer 3.}
\vspace{-8pt}     
\label{fig:surgery}
\end{figure}

\subsection{The inversion}

The spectral surgery applied to $S$--$D$ attention and standard attention produces opposite results across many metrics. In standard attention (Section~3, finding 1), filtering truncation is prohibitive and routing truncation is cheap, consistent with filtering dominance. In $S$--$D$ attention, routing truncation is prohibitive and filtering truncation is free; consistent with routing dominance and the scalar collapse of $D$.

This inversion confirms that the structure revealed by $S$--$D$ attention is architectural, not coincidental. The decomposition does not impose a routing-dominated structure, but permits one. And the structure that emerges is the opposite of what standard attention exhibits. The full spectral budget of GPT-2 Large (independent rank sweeps of routing and filtering, and the joint truncation grid) is reported in Appendix~\ref{app:budget}.

\begin{table}[!h]
\caption{Inversion summary. GPT-2 Large vs $S$--$D$ 125M.}
\label{tab:inversion}
\centering
\begin{tabular}{lcc}
\toprule
Surgery & GPT-2 Large & $S$--$D$ \\
\midrule
Remove all routing & PPL $985$ & PPL $699$ \\
Remove all filtering & PPL $12{,}034$ & PPL $81.0$ \\
Routing to rank $4$ & $+55\%$ & $+64\%$ \\
Filtering to rank $4$ & $+3{,}445\%$ & $+0.9\%$ \\
\bottomrule
\end{tabular}
\end{table}
The inversion lives in the filtering, rank-4 compression costs $+3,445\%$ in GPT-2 Large vs $+0.9\%$ in S–D revealing a 4,000× asymmetry.

\section{Spectrum-Guided Architecture}

The cascade prescribes architecture: which layers need how many dimensions, which can be linearized, and at what head dimension. We test two operationalizations: head-dim narrowing within softmax attention (Section~\ref{sec:param-realloc}) and replacing softmax with linear attention in early layers (Section~\ref{sec:compute-realloc}).

\subsection{Parameter reallocation}\label{sec:param-realloc}

The cascade prescribes per-layer head dimensions. The rank $\le 5$ plateau identified in Section~\ref{sec:routing-cascade} motivates $d_{\mathrm{head}}$ as small as $8$ in early layers; the high-rank tail at layers 8--11 needs full $d_{\mathrm{head}} = 64$. We instantiate this prescription by varying $d_{\mathrm{head}}$ per layer to match the measured spectral rank. This is standard $QK^{\top}$ attention compatible with FlashAttention [15], requiring no architectural modification beyond varying the projection dimension per layer. The $S$--$D$ model provides the diagnostic while standard attention is the deployment target. We train three architectures at this allocation and compare to the 125M Standard baseline (val PPL $27.87 \pm 0.01$).

\begin{table}[!h]
\caption{Cascade architecture configuration.}
\label{tab:cascade-arch}
\centering
\begin{tabular}{lccccc}
\toprule
Config & Total Params & Attn Params & Attn Savings & Val PPL & vs Standard \\
\midrule
Standard & 124M & $28.3$M & --- & $27.87 \pm 0.01$ & --- \\
Compressed & 106M & $10.0$M & $65\%$ & $30.20 \pm 0.06$ & $+8.4\%$ \\
Wide & 108M & $12.2$M & $57\%$ & $29.32 \pm 0.02$ & $+5.2\%$ \\
Deep & 125M & $15.0$M & $47\%$ & $28.97 \pm 0.08$ & $+3.9\%$ \\
\bottomrule
\end{tabular}
\end{table}

\paragraph{Compressed.} The diagnostic configuration: 12 layers, 12 heads per layer, $d_{\mathrm{head}}$ matched directly to the cascade $8/8/8/8/8/8/16/16/32/32/64/64$ across layers 0--11. Matching $d_{\mathrm{head}}$ to the spectral profile reduces attention parameters from $28.3$M to $10.0$M and total model parameters from $124$M to $106$M. Val perplexity is $30.20 \pm 0.06$. The $+8.4\%$ PPL gap measures how much the attention parameters in the standard uniform $d_{\mathrm{head}}=64$ architecture were doing useful work that pure cascade narrowing eliminates.

\paragraph{Wide.} Reinvest the saved budget as more heads in the cascade plateau. Layers 0--5 use $24$ heads $\times$ $d_{\mathrm{head}} = 8$ (double the head count, same per-head dimension), layers 6--7 use $16 \times d_{\mathrm{head}} = 16$, layers 8--9 use $12 \times d_{\mathrm{head}} = 32$, layers 10--11 use $12 \times d_{\mathrm{head}} = 64$. $108$M total parameters, $12.2$M attention ($-57\%$). Val PPL $29.32 \pm 0.02$. Each early-layer head captures one rotation plane; doubling the head count doubles routing capacity without restoring head dimensions the cascade says are unnecessary.

\paragraph{Deep.} Reinvest the saved budget in depth. $15$ layers with $d_{\mathrm{head}}$ allocation $8/8/8/8/8/8/8/16/16/32/32/64/64/64/64$. $125$M total parameters, $15.0$M attention ($-47\%$), parameter-matched to Standard. Val PPL $28.97 \pm 0.08$. Deep matched lands the closest of the three to baseline. Each additional cascade-compressed layer costs roughly an eighth of what a uniform-$d_{\mathrm{head}}=64$ layer costs in attention parameters while adding a full residual processing step.

The Compressed configuration is the diagnostic: spectral measurements from $S$--$D$ identify real redundancy in standard attention. Wide and Deep show two reinvestment strategies. Width is the more parameter-efficient while depth lands closest to baseline.

\subsection{Compute reallocation}\label{sec:compute-realloc}

A separate question from how many parameters each layer needs is whether early layers need $O(N^2)$ computation at all. If a layer uses only rank 2--3, the attention pattern it computes is low-dimensional enough that a feature-map approximation, $\phi(Q) \phi(K)^{\top}$ with $\phi(x) = \mathrm{elu}(x) + 1$ [16], computed via causal cumulative sums in $O(N)$, may be exact rather than approximate. Linear attention with simple feature maps [16] has been largely abandoned [19] on the assumption that the failure is intrinsic to the feature map. The cascade in Section~\ref{sec:routing-cascade} says otherwise: in the rank-2/3 plateau, the routing demand is low enough that the feature map's expressive ceiling is not the binding constraint. We test this directly. Replace the first $K$ layers with causal linear attention while keeping the remaining layers as softmax.

\begin{table}[!h]
\caption{Linear attention boundary sweep. First $k$ layers use ELU+1 causal linear attention ($O(N)$); remaining layers use softmax attention ($O(N^2)$). Cascade assigns $d_{\mathrm{head}}$ per layer matched to the measured spectral rank: $8$ for layers 0--5, $16$ for layers 6--7, $32$ for layers 8--9, $64$ for layers 10--11. Uniform is $d_{\mathrm{head}} = 64$ everywhere. Attention FLOPS counts projections and attention core at $2$ ops per multiply-add, summed across all $12$ layers per token at $N=1024$.}
\label{tab:linear-sweep}
\centering
\begin{tabular}{lcccccc}
\toprule
Config  & Attn Savings & Attn FLOPS Savings & Val PPL & vs Standard \\
\midrule
Standard & --- & --- & $27.87 \pm 0.01$ & --- \\
$K{=}4$ uniform & $0\%$ & $7.8\%$ & $28.25 \pm 0.09$ & $+1.4\%$ \\
$K{=}4$ cascade  & $64.6\%$ & $65.6\%$ & $30.38 \pm 0.03$ & $+9.0\%$ \\
$K{=}7$ uniform & $0\%$ & $13.7\%$ & $29.83$ & $+7.0\%$ \\
$K{=}7$ cascade  & $64.6\%$ & $66.6\%$ & $31.43 \pm 0.25$ & $+12.8\%$ \\
$K{=}10$ uniform & $0\%$ & $19.5\%$ & $31.15$ & $+11.8\%$ \\
$K{=}10$ cascade  & $64.6\%$ & $69.2\%$ & $33.15 \pm 0.11$ & $+18.9\%$ \\
\bottomrule
\end{tabular}
\end{table}

$K=4$ uniform reaches val PPL $28.25 \pm 0.09$ across two seeds, within $1.4\%$ of the all-quadratic baseline. ELU+1 linear attention works at full $d_{\mathrm{head}} = 64$ when applied to the layers the cascade identifies as rank 2--3. Extending linearization beyond the rank-2/3 plateau under uniform $d_{\mathrm{head}}=64$ degrades smoothly with the cumulative routing rank: $K=7$ costs $+7.0\%$ and $K=10$ costs $+11.8\%$ vs. baseline. The uniform sweep isolates the linearization cost itself and confirms it tracks the cascade. $K=4$ cascade narrows $d_{\mathrm{head}}$ to $8$ in the linear plateau, matching the spectral rank measurement. This costs an additional $+7.6\%$ PPL relative to $K=4$ uniform. Linearization and head-dim narrowing are separable design choices: linearizing cleanly is what the cascade prescribes, and head-dim narrowing within those layers trades parameters for perplexity on top. $K=4$ cascade reaches the same range as Compressed in Section~\ref{sec:param-realloc} (PPL $30.38$ vs $30.20$) at comparable attention parameters ($10.0$M each), but with $65.6\%$ fewer attention FLOPs. The cumulative cost rises monotonically with the cumulative routing rank of the linearized layers, both with and without head-dim narrowing. This is the Section~\ref{sec:perlayer-surgery} surgery result reproduced in models trained from scratch.

\begin{figure}[!t]
\centering
\begin{subfigure}{0.45\linewidth}
  \centering
  \includegraphics[height=4.5cm, keepaspectratio]{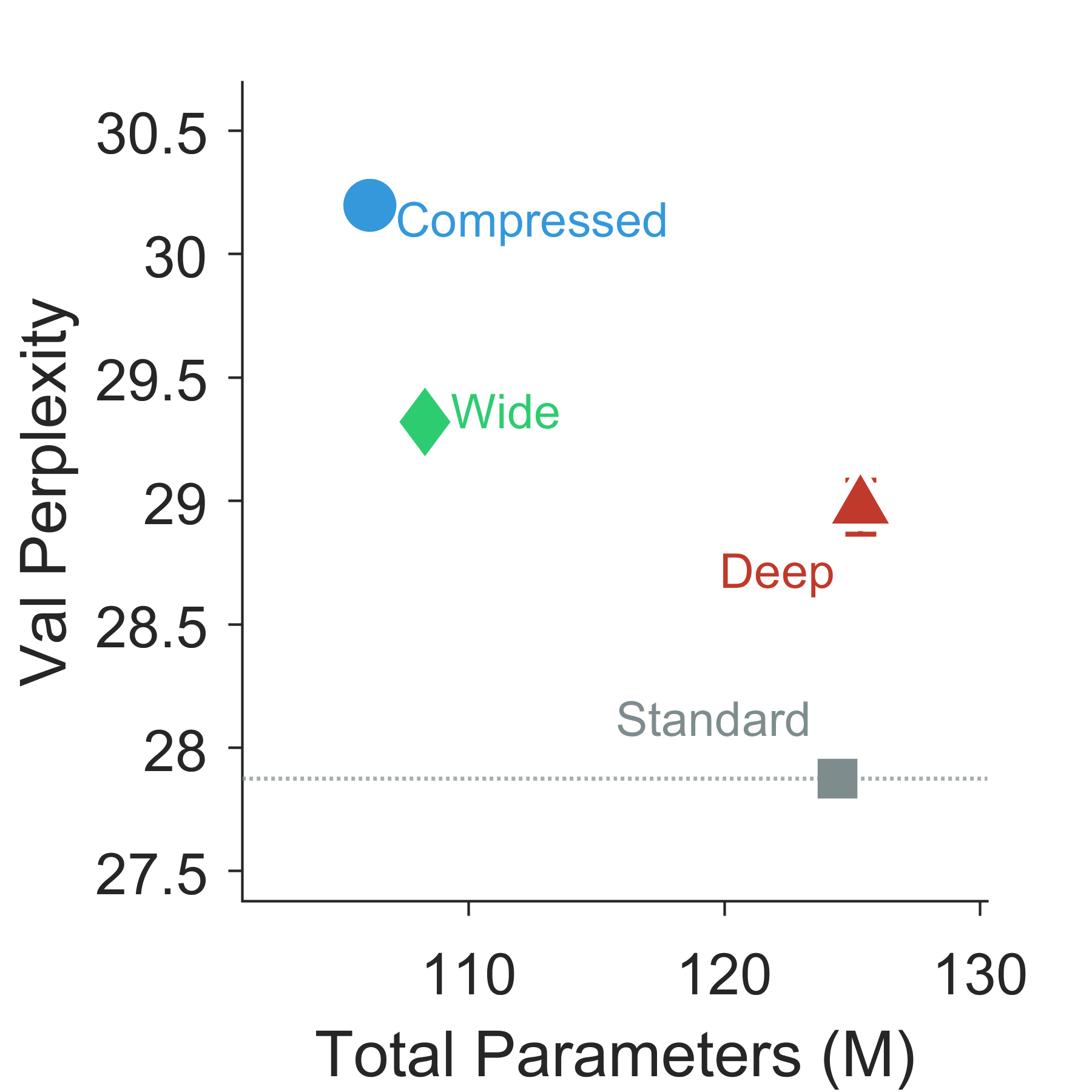}
\end{subfigure}
\hfill
\begin{subfigure}{0.45\linewidth}
  \centering
  \includegraphics[height=4.5cm, keepaspectratio]{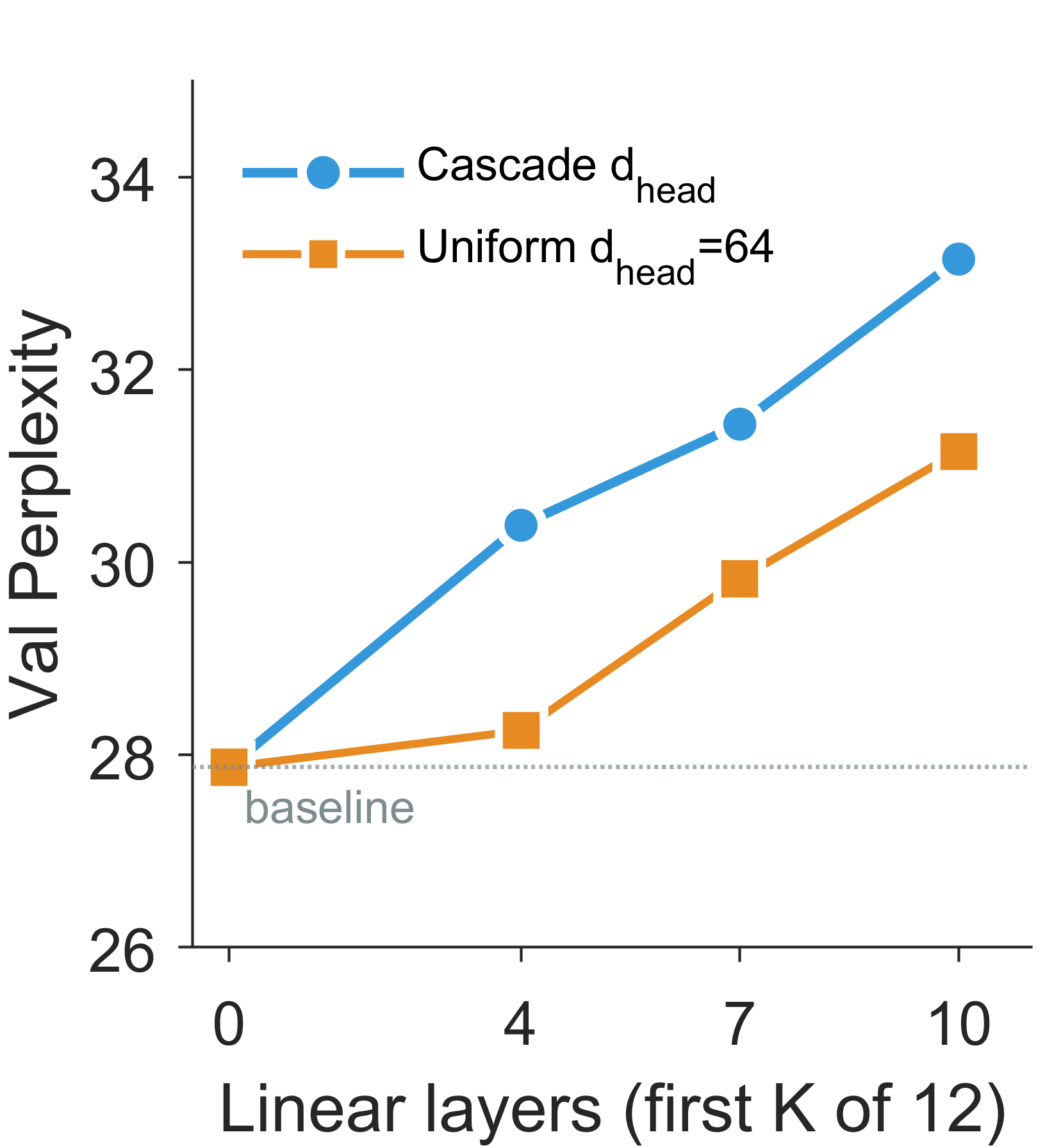}
\end{subfigure}
\caption{Spectrum-guided architecture. (a) Val perplexity vs.\ total parameters across four configurations. (b) Val perplexity vs.\ number of linearized layers (first $k$ of $12$ replaced with ELU+1 causal linear attention).}
\label{fig:arch}
\end{figure}

\section{Discussion}\label{sec:disc}

The contribution of this paper is the measurement, not the architecture. Every transformer has a spectral budget, a per-layer allocation of routing capacity, that standard attention hides by entangling routing and filtering. $S$--$D$ reveals it by separating them. The methodology: train a small $S$--$D$ diagnostic at the target depth, read the rank profile, prescribe accordingly. The cascade shape is not scale-invariant as the low-rank fraction grows from $33\%$ at 6 layers to $92\%$ at 24 layers. What is invariant is the direction: layer 0 converges toward rank $2$, terminal rank is set by head dimension, and the linearizable region widens with depth. A diagnostic at the target depth identifies where the boundary sits for that configuration. Validation at scale is the most important open experiment.

Hybrid architectures discover their attention/recurrence ratios by sweep: Nemotron-H lands at ${\sim}8\%$ full attention [2], Jamba at $1{:}7$ [3], Griffin and Zamba at different splits [4, 5]. Each works at its scale, none transfer cleanly. The cascade provides a continuous measurable rank that prescribes which layers need quadratic attention. At 355M, the cascade places $92\%$ of layers at rank $\le 5$, comparable to the empirical hybrid ratios.

Simple linear attention has been largely abandoned [19] on the assumption that the failure is intrinsic to the feature map. Our $K=4$ uniform result ($+1.4\%$ PPL at full $d_{\mathrm{head}}=64$) says otherwise: ELU+1 works in the layers the cascade identifies as rank 2--3, and degrades smoothly when extended into higher-rank layers ($+7.0\%$ at $K=7$, $+11.8\%$ at $K=10$). The failure was misallocation, not capacity. SSMs [7, 17] and gated recurrences [18, 20, 21] may extend the viable boundary further; the starting point for comparison should be the spectral rank of each layer, not the uniform assumption.

\paragraph{Limitations.} All experiments are done up to 355M on two epochs of OpenWebText. The cascade may shift quantitatively at larger scale. The $d_{\mathrm{head}}$ schedule derived from $S$--$D$ may not be exact for standard attention at every configuration. $S$--$D$ cannot use FlashAttention because it materializes the full $N \times N$ matrix which is acceptable for a diagnostic, not for deployment.

\begin{ack}
This work was partially supported by the Department of Energy Grant No.\ DE-SC0022248, Office of Naval Research Grant No.\ N00014-21-1-2634, and Air Force Office of Scientific Research Grants No.\ FA9550-21-1-0305 and FA9550-22-1-0433.
\end{ack}

\section*{References}

\medskip
{\small

[1] A.\ Vaswani, N.\ Shazeer, N.\ Parmar, J.\ Uszkoreit, L.\ Jones, A.\ N.\ Gomez, L.\ Kaiser, and I.\ Polosukhin.\ ``Attention is all you need,'' in \textit{Advances in Neural Information Processing Systems 30}, 2017.

[2] NVIDIA.\ ``Nemotron-H: A family of accurate and efficient hybrid Mamba-Transformer models,'' arXiv preprint arXiv:2504.03624, 2025.

[3] O. Lieber, \ B. Lenz, \ H. Bata, et al.\ ``Jamba: A hybrid transformer-Mamba language model,'' arXiv preprint arXiv:2403.19887, 2024.

[4] S.\ De, S.\ L.\ Smith, A.\ Fernando, A.\ Botev, et al.\ ``Griffin: Mixing gated linear recurrences with local attention for efficient language models,'' arXiv preprint arXiv:2402.19427, 2024.

[5] P. Glorioso, \ Q. Anthony, \ Y. Tokpanov, et al. \ ``Zamba: A compact 7B SSM hybrid model,'' arXiv preprint, arXiv:2405.16712, 2024.

[6] A.\ Botev, S.\ De, S.\ L.\ Smith, A.\ Fernando, et al.\ ``RecurrentGemma: Moving past transformers for efficient open language models,'' arXiv preprint, arXiv:2404.07839, 2024.

[7] A.\ Gu and T.\ Dao.\ ``Mamba: Linear-time sequence modeling with selective state spaces,'' arXiv preprint arXiv:2312.00752, 2023.

[8] A.\ Radford, J.\ Wu, R.\ Child, D.\ Luan, D.\ Amodei, and I.\ Sutskever.\ ``Language models are unsupervised multitask learners,'' OpenAI blog, 2019.

[9] J.\ Devlin, M.-W.\ Chang, K.\ Lee, and K.\ Toutanova.\ ``BERT: Pre-training of deep bidirectional transformers for language understanding,'' in \textit{Proceedings of NAACL-HLT}, 2019.

[10] J.\ L.\ Ba, J.\ R.\ Kiros, and G.\ E.\ Hinton.\ ``Layer normalization,'' arXiv preprint, arXiv:1607.06450, 2016.

[11] S.\ Jamil, and R. \ Kapadia.\ ``Interpretable Physics Extraction from Data for Linear Dynamical Systems using Lie Generator Networks,'' arXiv preprint, 	arXiv:2603.27442, 2026.

[12] S.\ Jamil, and R. \ Kapadia.\ ``Lie Generator Networks for Nonlinear Partial Differential Equations,'' arXiv preprint,  arXiv:2603.29264, 2026.

[13] A.\ Gokaslan and V.\ Cohen.\ OpenWebText corpus.\ \url{http://Skylion007.github.io/OpenWebTextCorpus}, 2019.

[14] S.\ Merity, C.\ Xiong, J.\ Bradbury, and R.\ Socher.\ ``Pointer sentinel mixture models,'' in \textit{International Conference on Learning Representations (ICLR)}, 2017.

[15] T.\ Dao, D.\ Y.\ Fu, S.\ Ermon, A.\ Rudra, and C.\ R\'e.\ ``FlashAttention: Fast and memory-efficient exact attention with IO-awareness,'' in \textit{Advances in Neural Information Processing Systems 35}, 2022.

[16] A.\ Katharopoulos, A.\ Vyas, N.\ Pappas, and F.\ Fleuret.\ ``Transformers are RNNs: Fast autoregressive transformers with linear attention,'' in \textit{Proceedings of the 37th International Conference on Machine Learning (ICML)}, 2020.

[17] A.\ Gu, K.\ Goel, and C.\ R\'e.\ ``Efficiently modeling long sequences with structured state spaces,'' in \textit{International Conference on Learning Representations (ICLR)}, 2022.

[18] S.\ Yang, B.\ Wang, Y.\ Shen, R.\ Panda, and Y.\ Kim.\ ``Gated linear attention transformers with hardware-efficient training,'' in \textit{Proceedings of the 41st International Conference on Machine Learning (ICML)}, 2024.

[19] K.\ Choromanski, V.\ Likhosherstov, D.\ Dohan, X.\ Song, et al.\ ``Rethinking attention with Performers,'' in \textit{International Conference on Learning Representations (ICLR)}, 2021.

[20] B.\ Peng, E.\ Alcaide, Q.\ Anthony, et al.\ ``RWKV: Reinventing RNNs for the transformer era,'' in \textit{Findings of the Association for Computational Linguistics: EMNLP}, 2023.

[21] Y.\ Sun, L.\ Dong, S.\ Huang, S.\ Ma, et al.\ ``Retentive network: A successor to transformer for large language models,'' arXiv preprint arXiv:2307.08621, 2023.

[22] S.\ Biderman, H.\ Schoelkopf, Q.\ Anthony, et al.\ ``Pythia: A suite for analyzing large language models across training and scaling,'' in \textit{Proceedings of the 40th International Conference on Machine Learning (ICML)}, 2023.
}

\appendix

\section{Evaluation Sequences and Decomposition Protocol}\label{app:sequences}

The pretrained model analysis in Section~3 uses six text sequences spanning different domains and lengths:
\begin{enumerate}
\item ``The cat sat on the mat and then it slowly walked to the door.'' (15 tokens)
\item ``Although the government proposed sweeping reforms to the healthcare system last year, the legislature has not yet passed any of the key provisions that were originally outlined in the draft bill submitted by the committee.'' (38 tokens)
\item ``In fluid dynamics, the Navier--Stokes equations describe the motion of viscous fluid substances. These partial differential equations arise from applying Newton's second law to fluid motion, together with the assumption that the stress in the fluid is the sum of a diffusing viscous term and a pressure term.'' (49 tokens)
\item ``She told him that the book he had lent her, which she had finally finished reading over the weekend, was one of the most thought-provoking novels she had encountered in years.'' (35 tokens)
\item ``The quick brown fox jumps over the lazy dog.'' (10 tokens)
\item ``Scientists at CERN announced that the particle accelerator had produced results consistent with theoretical predictions made decades ago, confirming that the Standard Model remains robust despite numerous attempts to find physics beyond it.'' (37 tokens)
\end{enumerate}
All $\rho$ and eigenvalue statistics are averaged across these six sequences per head. Token counts reflect GPT-2 BPE tokenization.

For autoregressive models (GPT-2 variants), the causal mask sets the upper triangle of $A$ to $-\infty$. This is an architectural constraint, not part of the learned interaction. Applying it before decomposition would flood $R$ and $F$ with mask artifacts. We decompose the unmasked $A$. For bidirectional models (BERT), no mask is applied and the distinction does not arise. Because GPT-2's weights were trained under a causal objective, where only the lower triangle of attention influences the loss, the learned $W_Q W_K^{\top}$ encodes different structure than BERT's weights, which were trained bidirectionally. The structural difference illustrated in Figure~\ref{fig:pretrained}(b) is a property of the learned interaction kernels, not an artifact of masking.

\section{The Spectral Budget of Attention}\label{app:budget}

We take a pretrained GPT-2 Large (baseline perplexity $14.11$ on WikiText-2 val [14]), intercept the attention interaction matrix at runtime, decompose into $R + F$, truncate each component independently to a target rank via SVD, and evaluate perplexity. No retraining is done.

\begin{figure}[!h]
\centering
\begin{subfigure}{0.4\linewidth}
  \centering
  \includegraphics[height=4.5cm, keepaspectratio]{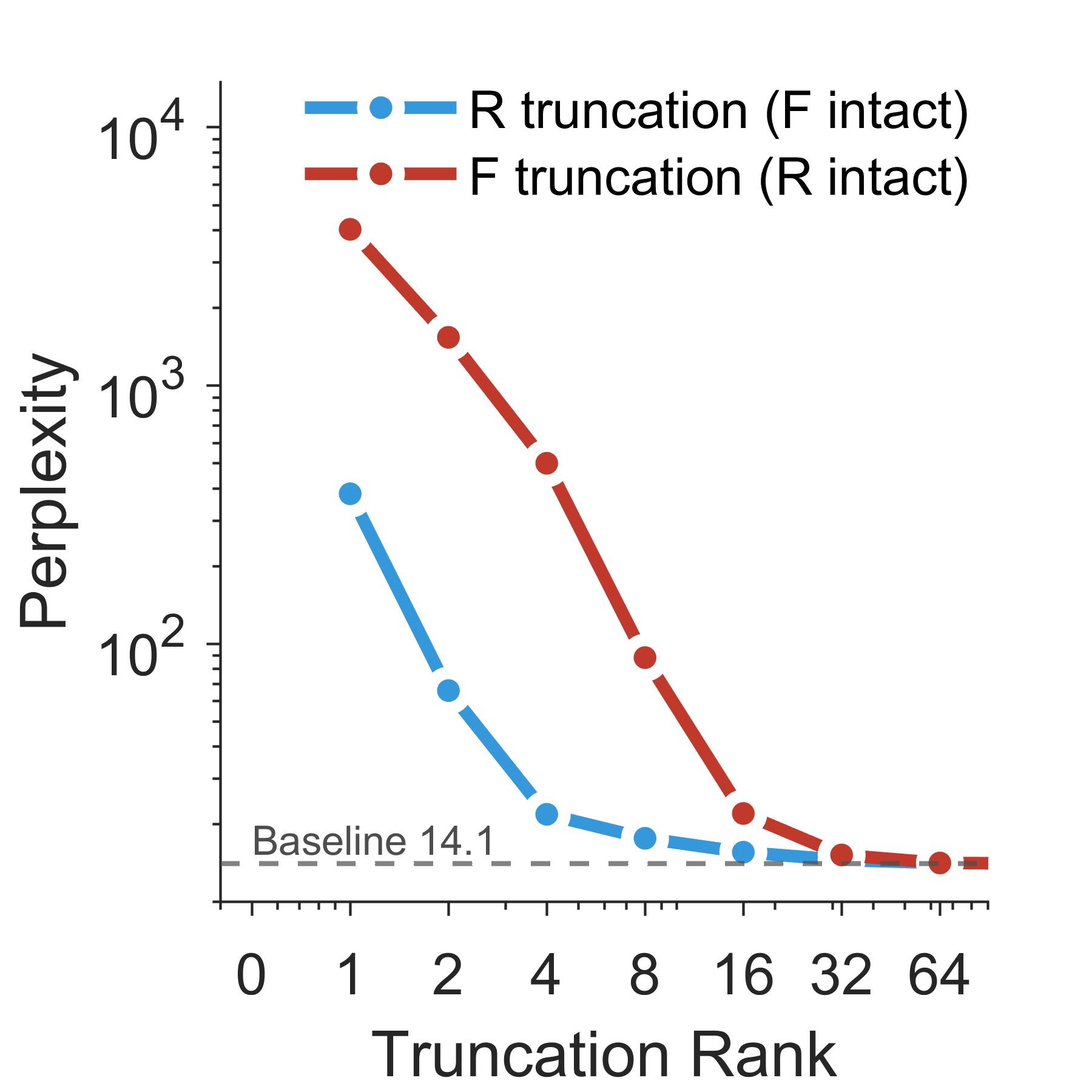}
\end{subfigure}
\hfill
\begin{subfigure}{0.5\linewidth}
  \centering
  \includegraphics[height=4.5cm, keepaspectratio]{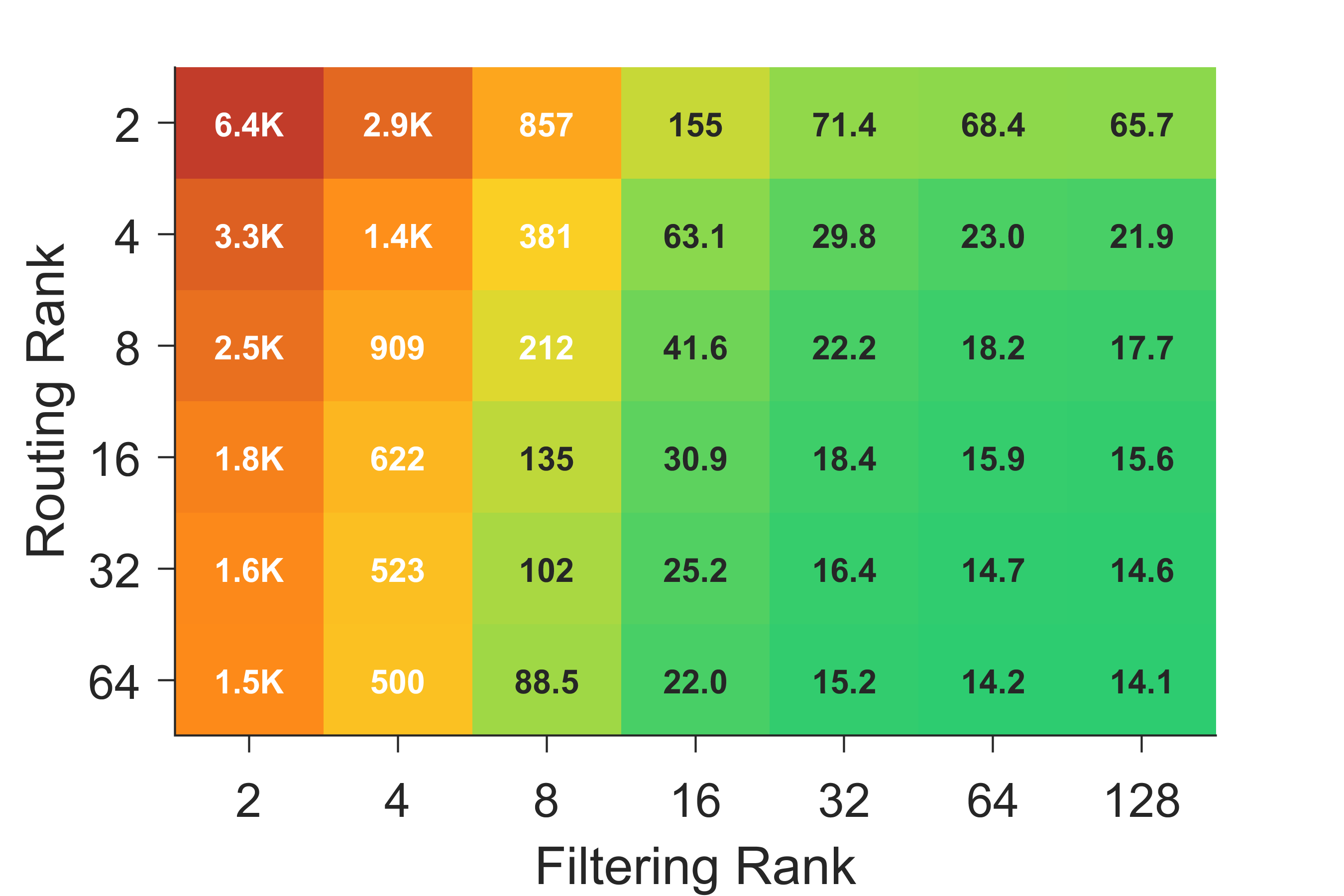}
\end{subfigure}
\caption{The spectral budget of attention in GPT-2 Large (WikiText-2 val, baseline PPL $14.1$). (a) Perplexity after independently truncating routing (blue) and filtering (red) to a given rank while keeping the counterpart full rank. (b) Perplexity under joint truncation across all routing--filtering rank combinations.}
\label{fig:budget}
\end{figure}

Removing all routing ($R = 0$, filtering intact) yields PPL $985$. However, removing all filtering ($F = 0$, routing full rank) yields PPL $12{,}034$ which is ${\sim}12\times$ worse. The asymmetry persists across the full rank sweep (Fig.~\ref{fig:budget}a). At rank $8$, filtering truncation gives PPL $88$ while routing truncation gives PPL $18$. At rank $32$, filtering gives PPL $15.2$ ($+7.7\%$) while routing gives PPL $14.6$ ($+3.1\%$). The component that $\rho$ identifies as dominant is functionally dominant. Routing converges at rank $32$ while filtering converges at rank $64$ beyond which the improvement is marginal. Figure~\ref{fig:budget}(b) shows this directly: near-baseline perplexity occupies a small corner of the rank space.

This redundancy is a consequence of entangling routing and filtering in a single unconstrained matrix. The model cannot organize its routing across depth because it has no structural separation between the two functions. Therefore, it compensates by overinvesting in filtering.

\section{The Diagonal Offset \texorpdfstring{$\varepsilon$}{epsilon}}\label{app:epsilon}

The $S$--$D$ parameterization $L = S - D$ constrains $\mathrm{Re}(\lambda) \le 0$ for any $D \succeq 0$. We parameterize
\begin{equation*}
d_i = \mathrm{softplus}(W_d \cdot x_i + b_d) + \varepsilon
\end{equation*}
with $b_d$ a learnable per-head vector and $\varepsilon$ a fixed offset. The role of $\varepsilon$ is to keep eigenvalues away from the imaginary axis during training, where the eigenvalue problem becomes ill-conditioned. The offset does not affect the structural guarantee, but its absence in coordinates where $\mathrm{softplus}(W_d \cdot x_i + b_d)$ approaches zero leads to optimization failure. We trained the $S$--$D$ no-LN configuration at four values of $\varepsilon$ on 125M parameter models. Outcomes are summarized below.
\begin{figure}[!t]
\centering
\includegraphics[width=0.4\linewidth]{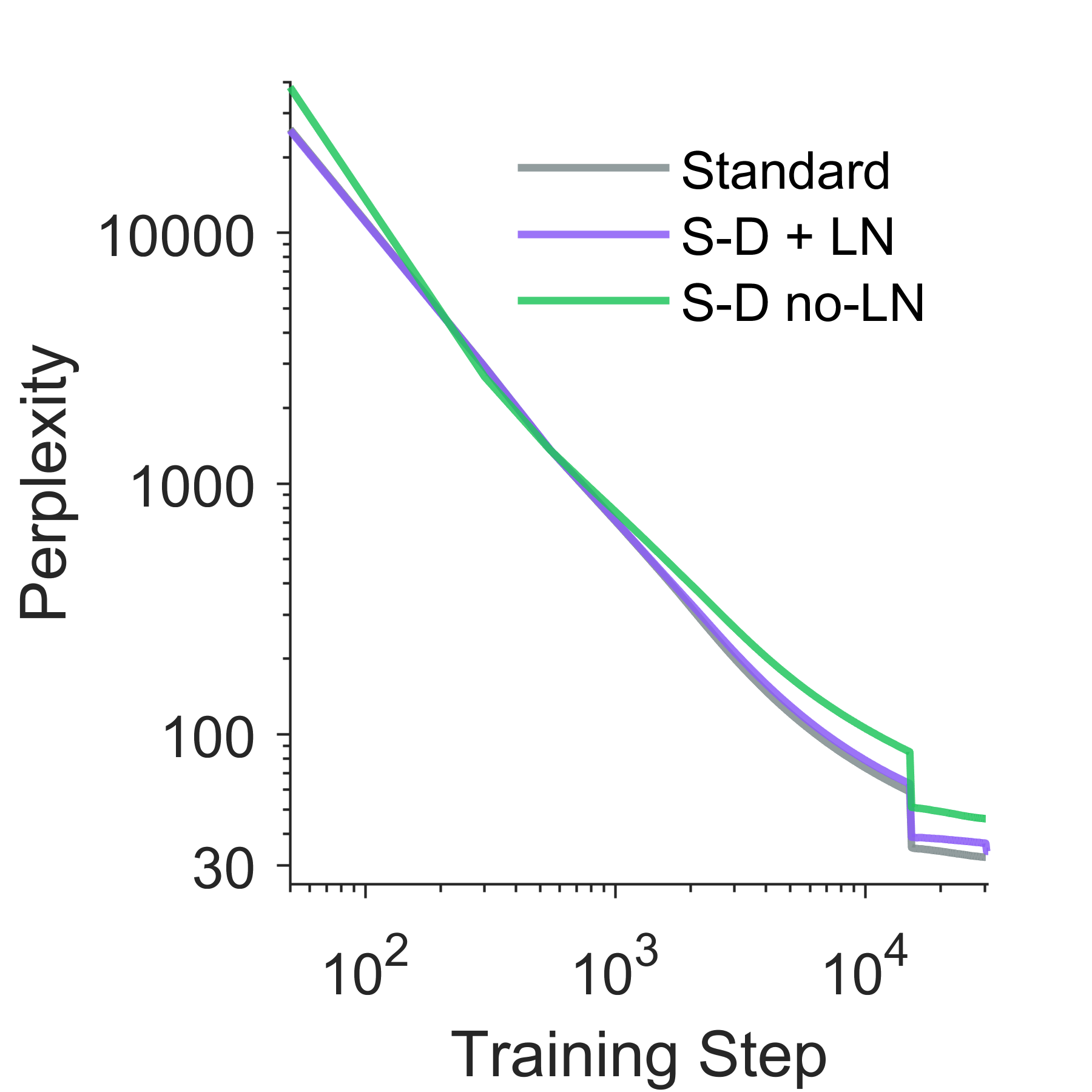}
\caption{$S$--$D$ attention at 125M parameters (OpenWebText). Training perplexity. $S$--$D$ without LayerNorm trains stably over 28{,}000 steps with no divergence.}
\label{fig:training}
\end{figure}

\begin{table}[!h]
\caption{$S$--$D$ attention convergence outcomes across $\varepsilon$ values.}
\label{tab:epsilon}
\centering
\begin{tabular}{ccccc}
\toprule
$\varepsilon$ & Seeds & Converged & Diverged & Val PPL \\
\midrule
$0$ & $3$ & $0$ & $3$ & --- \\
$0.01$ & $3$ & $1$ & $2$ & $37.5$ ($n{=}1$) \\
$0.05$ & $3$ & $3$ & $0$ & $37.39 \pm 0.09$ \\
$0.10$ & $1$ & $1$ & $0$ & $37.38$ \\
\bottomrule
\end{tabular}
\end{table}

Without an offset, all three seeds developed NaN losses during training: at optimizer steps $3{,}350$, $8{,}450$, and $18{,}000$. The variability in the divergence step suggests the failure is not purely deterministic; it depends on the random trajectory through parameter space. At $\varepsilon = 0.01$, one seed converged to PPL $37.5$; one seed produced loss values that grew to ${\sim}28$ with PPL ${\sim}5 \times 10^8$ (a numerical blow-up rather than NaN); one seed produced NaN at step $5{,}550$. At $\varepsilon = 0.05$ and above, all attempted seeds converged. The three $\varepsilon = 0.05$ seeds reached val PPL $37.29$, $37.37$, and $37.50$ (mean $37.39$, $\sigma=0.09$).

When $\mathrm{softplus}(W_d \cdot x_i + b_d)$ approaches zero in some heads at some training steps, the eigenvalues of $L = S - D$ approach the imaginary axis, where the eigendecomposition becomes ill-conditioned. The result is that small perturbations propagate into large parameter updates, eventually producing NaN. The offset $\varepsilon$ establishes a fixed margin from the imaginary axis that absorbs gradient noise.

The choice of $\varepsilon = 0.05$ in the main text is the smallest swept value at which all seeds converged. Higher values ($\varepsilon = 0.10$) also converged with similar final PPL but were not swept further; we did not sweep below $\varepsilon = 0.01$ or characterize the precise threshold between marginal stability ($\varepsilon = 0.01$) and reliable stability ($\varepsilon = 0.05$). Higher values would be expected to bias the model toward over-damping and degrade performance, but this regime was not tested.

\section{Emergence of the Cascade}\label{app:emergence}

\begin{figure}[!h]
\centering
\begin{subfigure}{0.45\linewidth}
  \centering
  \includegraphics[height=5cm, keepaspectratio]{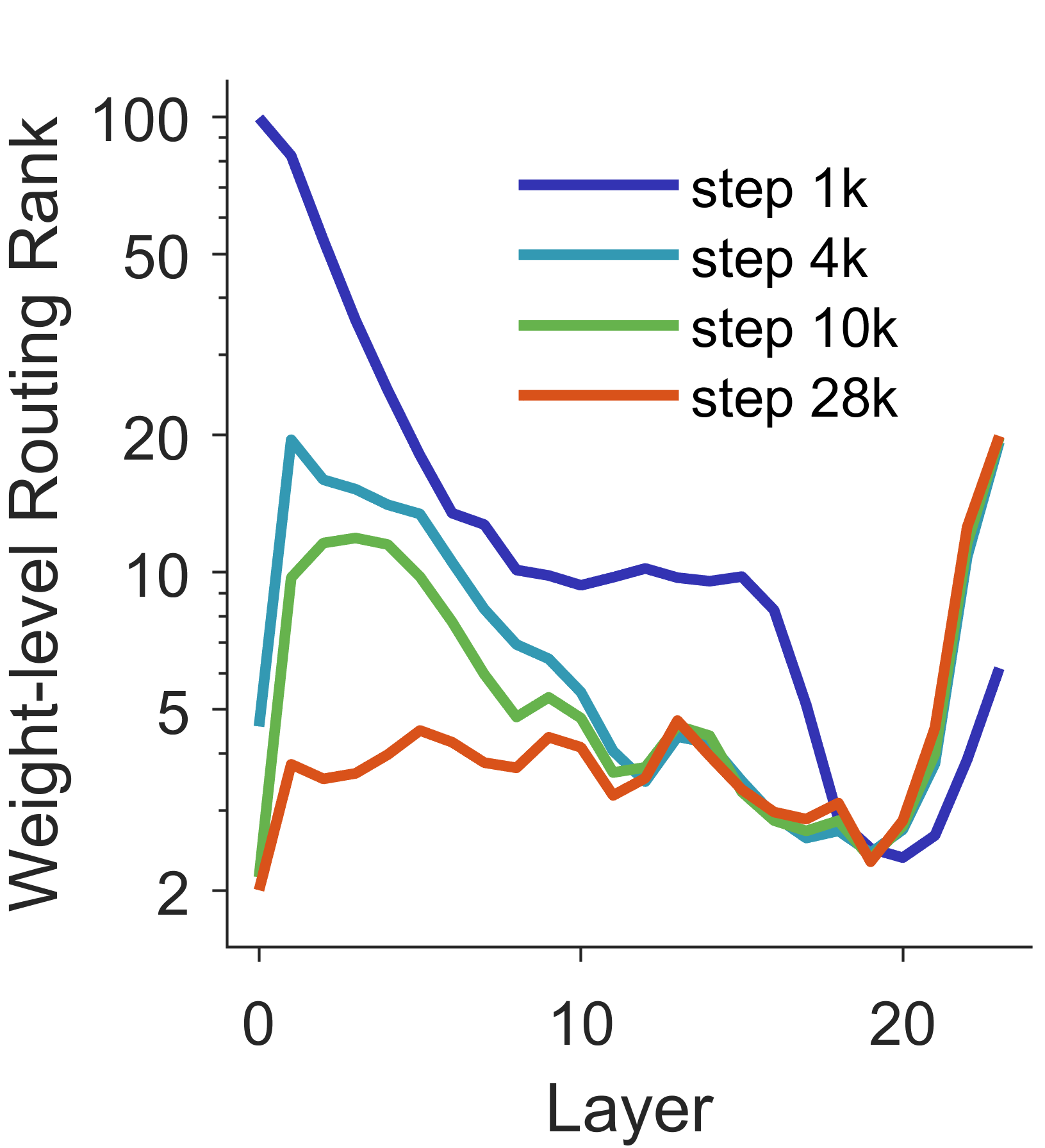}
\end{subfigure}
\hfill
\begin{subfigure}{0.45\linewidth}
  \centering
  \includegraphics[height=5cm, keepaspectratio]{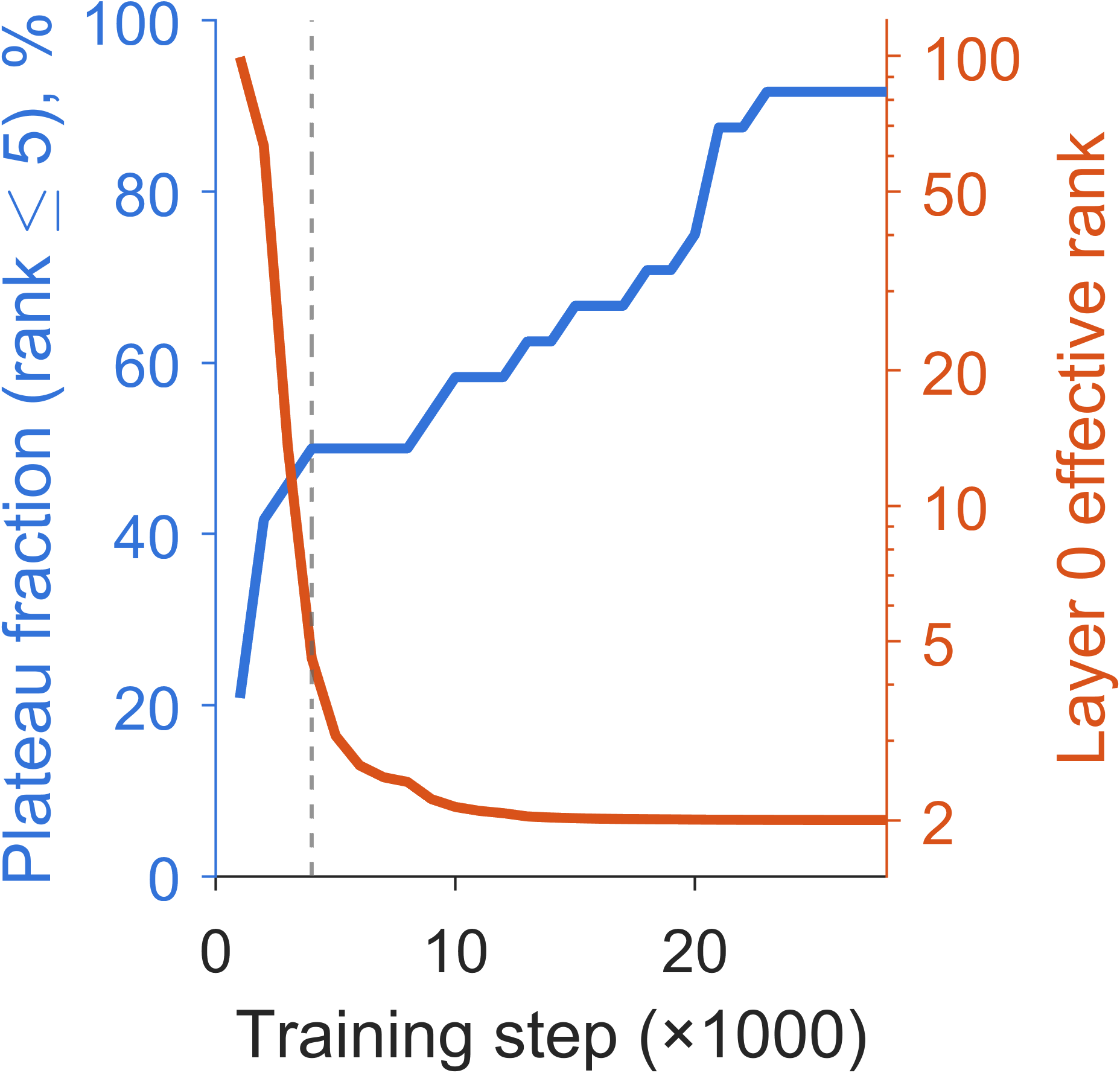}
\end{subfigure}
\caption{Cascade emergence during 355M training. (a) Per-layer effective routing rank at log-spaced training steps ($1$k, $4$k, $10$k, $28$k), showing shape formation and low-rank widening. (b) Low-rank fraction (layers at rank $\le 5$) and layer-0 rank vs.\ training step. The terminal rank converges by step ${\sim}4$k; the plateau widens from $50\%$ to $92\%$ over the subsequent $24$k steps.}
\label{fig:emergence}
\end{figure}

The cascade in $S$--$D$ attention is not present at initialization. It emerges over training in two distinguishable phases. We log the per-layer effective routing rank every $1{,}000$ steps during the 355M $S$--$D$ no-LN training run ($28{,}000$ steps total) and report the trajectory in Fig.~\ref{fig:emergence}.

\paragraph{Phase 1: shape formation (steps 1k--4k).} At step $1{,}000$, layer 0 effective rank is ${\sim}100$; the projections have been updated only briefly and the per-head kernel is still close to its random initialization. By step $4{,}000$, layer 0 has collapsed to rank $4.58$ and the terminal layer has reached its final amplitude (rank ${\sim}20$). The qualitative shape of the cascade, low at the front, high at the back, is established within the first phase of training.

\paragraph{Phase 2: low-rank widening (steps 4k--28k).} The terminal ranks are essentially stationary after phase 1. The fraction of layers operating at rank $\le 5$ grows monotonically: $50\%$ at step $4$k, $58\%$ at step $10$k, $75\%$ at step $20$k, $92\%$ by step $23$k (where it stabilizes).

The two phases describe complementary processes. Phase 1 establishes that the optimizer organizes routing into a depth-asymmetric shape; a few high-rank layers at the back, low-rank layers elsewhere. Phase 2 progressively consolidates the low-rank region. As training continues, the optimizer discovers that successive layers can operate at minimum rank without PPL increase, and the high-rank work retreats toward the terminal layer. The final $92\%$ plateau is not the initial decomposition; it is the converged result of the optimizer compressing routing into the smallest depth slice that the architecture allows.

This trajectory clarifies the Section~\ref{sec:routing-cascade} finding. The cascade is not a property of the parameterization that snaps into place, it is the equilibrium the optimizer arrives at when filtering and depth-equalization no longer obstruct the natural depth asymmetry. The shape emerges fast; the plateau widens slowly.

\end{document}